%% file: doubleQ.tex
\documentclass{article}

\usepackage[final]{neurips_2020}

\usepackage[utf8]{inputenc} 
\usepackage[T1]{fontenc}    
\usepackage{hyperref}       
\usepackage{url}            
\usepackage{booktabs}       
\usepackage{amsfonts}       
\usepackage{nicefrac}       
\usepackage{microtype}      
\hypersetup{
 colorlinks,linkcolor=magenta,anchorcolor=blue,citecolor=blue}

\usepackage{graphicx}
\usepackage{subfigure}
\usepackage{amsmath} 
\usepackage{amssymb}  
\usepackage{amsthm} 
\usepackage{color}
\usepackage{xcolor}
\usepackage{caption}
\newcommand{\norm}[1]{\left\lVert#1\right\rVert}
\usepackage{cleveref}

\makeatletter
\newtheorem*{rep@theorem}{\rep@title}
\newcommand{\newreptheorem}[2]{%
\newenvironment{rep#1}[1]{%
 \def\rep@title{#2 \ref{##1}}%
 \begin{rep@theorem}}%
 {\end{rep@theorem}}}
\makeatother

\newtheorem{theorem}{Theorem}

\newtheorem{lemma}{Lemma}

\newtheorem{definition}{Definition}
\newtheorem{proposition}{Proposition}
\newreptheorem{proposition}{Proposition}
\newreptheorem{theorem}{Theorem}
\newtheorem{assumption}{Assumption}
\usepackage{algorithm}
\usepackage{algorithmic}

\usepackage{multirow}

\usepackage{dirtytalk}

\newcommand{\mP}{\mathbb P}
\newcommand{\mE}{\mathbb E}
\newcommand{\mR}{\mathbb R}
\newcommand{\mcF}{\mathcal F}
\newcommand{\mcs}{\mathcal S}
\newcommand{\mca}{\mathcal A}

\newcommand{\mcT}{\mathcal T}

\allowdisplaybreaks[3]

\usepackage{authblk}

\usepackage{apalike}

\title{Finite-Time Analysis for Double Q-learning}

\begin{document}

\author[1]{Huaqing Xiong}
\author[2]{Lin Zhao}
\author[1]{Yingbin Liang}
\author[ .3,4]{Wei Zhang\thanks{Corresponding author}}
\affil[1]{The Ohio State University}
\affil[2]{National University of Singapore}
\affil[3]{Southern University of Science and Technology}
\affil[4]{Peng Cheng Laboratory}
\affil[ ]{\textsuperscript{1}\{xiong.309, liang.889\}@osu.edu;\quad \textsuperscript{2}elezhli@nus.edu.sg;\quad \textsuperscript{3,4}zhangw3@sustech.edu.cn}

\maketitle

\begin{abstract}
Although Q-learning is one of the most successful algorithms for finding the best action-value function (and thus the optimal policy) in reinforcement learning, its implementation often suffers from large overestimation of Q-function values incurred by random sampling. The double Q-learning algorithm proposed in~\citet{hasselt2010double} overcomes such an overestimation issue by randomly switching the update between two Q-estimators, and has thus gained significant popularity in practice. However, the theoretical understanding of double Q-learning is rather limited. So far only the asymptotic convergence has been established, which does not characterize how fast the algorithm converges. In this paper, we provide the first non-asymptotic (i.e., finite-time) analysis for double Q-learning. We show that both synchronous and asynchronous double Q-learning are guaranteed to converge to an $\epsilon$-accurate neighborhood of the global optimum by taking
$\tilde{\Omega}\left(\left( \frac{1}{(1-\gamma)^6\epsilon^2}\right)^{\frac{1}{\omega}} +\left(\frac{1}{1-\gamma}\right)^{\frac{1}{1-\omega}}\right)$ iterations, where $\omega\in(0,1)$ is the decay parameter of the learning rate, and $\gamma$ is the discount factor.
Our analysis develops novel techniques to derive finite-time bounds on the difference between two inter-connected stochastic processes, which is new to the literature of stochastic approximation.
\end{abstract}

\section{Introduction}

Q-learning is one of the most successful classes of reinforcement learning (RL) algorithms, which aims at finding the optimal action-value function or Q-function (and thus the associated optimal policy) via off-policy data samples. The Q-learning algorithm was first proposed by \citet{watkins1992q}, 
and since then, it has been widely used in various applications including robotics~\citep{tai2016robot}, autonomous driving~\citep{okuyama2018autonomous}, video games~\citep{mnih2015human}, to name a few. Theoretical performance of Q-learning has also been intensively explored. The asymptotic convergence has been established in \citet{tsitsiklis1994asynchronous,jaakkola1994convergence,borkar2000ode,melo2001convergence,Lee2019Switch}. The non-asymptotic (i.e., finite-time) convergence rate of Q-learning was firstly obtained in \citet{szepesvari1998asymptotic}, and has been further studied in~\citep{even2003learning,shah2018q,wainwright2019stochastic,beck2012error,Chen2020finiteSample} for synchronous Q-learning and in ~\citep{even2003learning,qu2020finite} for asynchoronous Q-learning. 



One major weakness of Q-learning arises in practice due to the large overestimation of the action-value function~\citep{hasselt2010double,van2016deep}. Practical implementation of Q-learning involves using the maximum {\em sampled} Q-function to estimate the maximum \textit{expected} Q-function (where the expectation is taken over the randomness of reward). Such an estimation often yields a large positive bias error~\citep{hasselt2010double}, and causes Q-learning to perform rather poorly.
To address this issue, double Q-learning was proposed in~\citet{hasselt2010double}, which keeps two Q-estimators (i.e., estimators for Q-functions), one for estimating the maximum Q-function value and the other one for update, and continuously changes the roles of the two Q-estimators in a random manner. It was shown in \citet{hasselt2010double} that such an algorithm effectively overcomes the overestimation issue of the vanilla Q-learning. In \citet{van2016deep}, double Q-learning was further demonstrated to substantially improve the performance of Q-learning with deep neural networks (DQNs) for playing Atari 2600 games. It inspired many variants~\citep{zhang2017weighted,abed2018double}, received a lot of applications~\citep{zhang2018double,zhang2018human}, and have become one of the most common techniques for applying Q-learning type of algorithms~\citep{hessel2018rainbow}.

Despite its tremendous empirical success and popularity in practice, theoretical understanding of double Q-learning is rather limited. Only the asymptotic convergence was provided in \citet{hasselt2010double,weng2020provably}. There has been no non-asymptotic result on how fast double Q-learning converges. From the technical standpoint, such finite-time analysis for double Q-learning does not follow readily from those for the vanilla Q-learning, because it involves two randomly updated Q-estimators, and the coupling between these two random paths significantly complicates the analysis. This goes much more beyond the existing techniques for analyzing the vanilla Q-learning that handles the random update of a single Q-estimator. Thus, {\em the goal of this paper is to develop new finite-time analysis techniques that handle the inter-connected two random path updates in double Q-learning and provide the convergence rate.}

\vspace{-1em}
\subsection{Our contributions}

The main contribution of this paper lies in providing the first finite-time analysis for double Q-learning with both the synchronous and asynchronous implementations. 
\begin{list}{$\bullet$}{\topsep=0.ex \leftmargin=0.3in \rightmargin=0.in \itemsep =0.02in}

\item We show that synchronous double Q-learning with a learning rate $\alpha_t = 1/t^\omega$ (where $\omega\in(0,1)$) attains an $\epsilon$-accurate global optimum with at least the probability of $1-\delta$ by taking $\Omega\left( \left( \frac{1}{(1-\gamma)^6\epsilon^2}\ln \frac{|\mcs||\mca|}{(1-\gamma)^7\epsilon^2\delta} \right)^{\frac{1}{\omega}} + \left(\frac{1}{1-\gamma} \ln\frac{ 1}{(1-\gamma)^2\epsilon} \right)^{\frac{1}{1-\omega}} \right)$ iterations, where $\gamma\in(0,1)$ is the discount factor, 
$|\mcs|$ and $|\mca|$ are the sizes of the state space and action space, respectively. 

\item We further show that under the same accuracy and high probability requirements, asynchronous double Q-learning takes $\Omega\left( \left( \frac{L^4}{(1-\gamma)^6\epsilon^2}\ln \frac{|\mcs||\mca|L^4}{(1-\gamma)^7\epsilon^2\delta} \right)^{\frac{1}{\omega}} + \left(\frac{L^2}{1-\gamma} \ln\frac{ 1}{(1-\gamma)^2\epsilon} \right)^{\frac{1}{1-\omega}} \right)$ iterations, where $L$ is the covering number specified by the exploration strategy. 
\end{list}

Our results corroborate the design goal of double Q-learning, which opts for better accuracy by making less aggressive progress during the execution in order to avoid overestimation. Specifically, our results imply that in the high accuracy regime, double Q-learning achieves the same convergence rate as vanilla Q-learning in terms of the order-level dependence on $\epsilon$, which further indicates that the high accuracy design of double Q-learning dominates the less aggressive progress in such a regime. In the low-accuracy regime, which is not what double Q-learning is designed for, the cautious progress of double Q-learning yields a slightly weaker convergence rate than Q-learning in terms of the dependence on $1-\gamma$. 

From the technical standpoint, our proof develops new techniques beyond the existing finite-time analysis of the vanilla Q-learning with a single random iteration path. 
More specifically, we model the double Q-learning algorithm as two alternating stochastic approximation (SA) problems, where one SA captures the error propagation between the two Q-estimators, and the other captures the error dynamics between the Q-estimator and the global optimum. For the first SA, we develop new techniques to provide the finite-time bounds on the two inter-related stochastic iterations of Q-functions. Then we 
develop new tools to bound the convergence of Bernoulli-controlled stochastic iterations of the second SA conditioned on the first SA. 


\subsection{Related work}

Due to the rapidly growing literature on Q-learning, we review only the theoretical results that are highly relevant to our work.

Q-learning was first proposed in \citet{watkins1992q} under finite state-action space. Its asymptotic convergence has been established in \citet{tsitsiklis1994asynchronous,jaakkola1994convergence,borkar2000ode,melo2001convergence} through studying various general SA algorithms that include Q-learning as a special case. Along this line, \citet{Lee2019Switch} characterized Q-learning as a switched linear system and applied the results of~\citet{borkar2000ode} to show the asymptotic convergence, which was also extended to other Q-learning variants. Another line of research focuses on the finite-time analysis of Q-learning which can capture the convergence rate. Such non-asymptotic results were firstly obtained in \citet{szepesvari1998asymptotic}. A more comprehensive work~\citep{even2003learning} provided finite-time results for both synchronous and asynchoronous Q-learning. Both \citet{szepesvari1998asymptotic} and \citet{even2003learning} showed that with linear learning rates, the convergence rate of Q-learning can be exponentially slow as a function of $\frac{1}{1-\gamma}$. To handle this, the so-called rescaled linear learning rate was introduced to avoid such an exponential dependence in synchronous Q-learning~\citep{wainwright2019stochastic,Chen2020finiteSample} and asynchronous Q-learning~\citep{qu2020finite}. The finite-time convergence of Q-learning was also analyzed with constant step sizes~\citep{beck2012error,Chen2020finiteSample,li2020sample}. Moreover, the polynomial learning rate, which is also the focus of this work, was investigated for both synchronous~\citep{even2003learning,wainwright2019stochastic} and asynchronous Q-learning~\citep{even2003learning}. In addition, it is worth mentioning that \cite{shah2018q} applied the nearest neighbor approach to handle MDPs on infinite state space.

Differently from the above extensive studies of vanilla Q-learning, theoretical understanding of double Q-learning is limited. The only theoretical guarantee was on the asymptotic convergence provided by \citet{hasselt2010double,weng2020provably}, which do not provide the non-asymptotic (i.e., finite-time) analysis on how fast double Q-learning converges. 
This paper provides the first finite-time analysis for double Q-learning.

The vanilla Q-learning algorithm has also been studied for the function approximation case, i.e., the Q-function is approximated by a class of parameterized functions. In contrast to the tabular case, even with linear function approximation, Q-learning has been shown not to converge in general \citep{baird1995residual}. Strong assumptions are typically imposed to guarantee the convergence of Q-learning with function approximation \citep{bertsekas1996neuro,zou2019finite,Chen2019finiteQ,du2019provably,xu2019deepQ,cai2019neural,weng2020analysis,weng2020momentum}. Regarding double Q-learning, it is still an open topic on how to design double Q-learning algorithms under function approximation and under what conditions they have theoretically guaranteed convergence.

\section{Preliminaries on Q-learning and Double Q-learning}

In this section, we introduce the Q-learning and the double Q-learning algorithms.

\subsection{Q-learning}

We consider a $\gamma$-discounted Markov decision process (MDP) with a finite state space $\mcs$ and a finite action space $\mca$. The transition probability of the MDP is given by $P:\mcs \times \mca \times \mcs \rightarrow [0,1]$, that is, $\mathbb{P}(\cdot|s, a)$ denotes the probability distribution of the next state given the current state $s$ and action $a$. We consider a random reward function $R_t$ at time $t$ drawn from a fixed distribution $\phi: \mathcal{S}\times \mathcal{A}\times \mathcal{S} \mapsto \mathbb{R}$, where $\mathbb{E}\{R_t(s,a,s')\}=R_{sa}^{s'}$ and $s'$ denotes the next state starting from $(s,a)$. In addition, we assume $|R_t|\leq R_{\max}$.
A policy $\pi:=\pi(\cdot|s)$ characterizes the conditional probability distribution over the action space $\mathcal{A}$ given each state $s\in\mathcal{S}$. 

The action-value function (i.e., Q-function) $Q^{\pi}\in\mR^{|\mcs|\times |\mca|}$ for a given policy $\pi$ is defined as
\begin{align}\label{eq:Qfunction}
    Q^{\pi}(s,a):=&\mE\left[\sum_{t=0}^{\infty}\gamma^t R_t(s,\pi(s),s')\Big|s_0=s,a_0=a \right] \nonumber \\
    =& \mE_{\substack{s'\sim P(\cdot|s,a)\\a'\sim\pi(\cdot|s')}} \left[R_{sa}^{s'}+\gamma Q^{\pi}(s',a')\right],
\end{align}
where $\gamma\in(0,1)$ is the discount factor. Q-learning aims to find the Q-function of an optimal policy $\pi^*$ that maximizes the accumulated reward. The existence of such a $\pi^*$ has been proved in the classical MDP theory~\citep{bertsekas1996neuro}. The corresponding optimal Q-function, denoted as $Q^*$, is known as the unique fixed point of the Bellman operator $\mcT$ given by
\begin{equation}\label{eq:BellmanOperator}
    \mcT Q(s,a) = \mE_{s'\sim P(\cdot|s,a)} \left[R_{sa}^{s'}+\gamma \underset{a' \in U(s')}{\max}Q(s',a')\right],
\end{equation}
where $ U(s')\subset\mca$ is the admissible set of actions at state $s'$. It can be shown that the Bellman operator $\mcT$ is $\gamma$-contractive in the supremum norm $\norm{Q}:=\max_{s,a}|Q(s,a)|$, i.e., it satisfies
\begin{equation}
    \label{eq:Contraction}
    \norm{\mcT Q - \mcT Q'} \leq \gamma\norm{Q - Q'}.
\end{equation}
The goal of Q-learning is to find $Q^*$, which further yields $\pi^*(s) = \arg\max_{a\in U(s)}Q^*(s,a)$.
In practice, however, exact evaluation of the Bellman operator~\eqref{eq:BellmanOperator} is usually infeasible due to the lack of knowledge of the transition kernel of MDP and the randomness of the reward. Instead, Q-learning draws random samples to estimate the Bellman operator and iteratively learns $Q^*$ as
\begin{equation}\label{eq:qlearning}
    Q_{t+1}(s,a) = (1-\alpha_t(s,a)) Q_t(s,a) + \alpha_t(s,a) \left( R_t(s,a,s') + \gamma\underset{a'\in U(s')}{\max}Q_t(s',a') \right),
\end{equation}
where $R_t$ is the sampled reward, $s'$ is sampled by the transition probability given $(s,a)$, and $\alpha_t(s,a)\in(0,1]$ denotes the learning rate.

\subsection{Double Q-learning}

Although Q-learning is a commonly used RL algorithm to find the optimal policy, it can suffer from overestimation in practice~\citep{smith2006optimizer}. To overcome this issue, \citet{hasselt2010double} proposed double Q-learning given in Algorithm~\ref{alg:doubleQ}.

\begin{algorithm}[H]
 	\caption{Synchronous Double Q-learning~\citep{hasselt2010double}} \label{alg:doubleQ} 
 	\begin{algorithmic}[1]
 		\STATE 	{\bf Input:} Initial $Q^A_1, Q^B_1$.
		\FOR{$t=1,2,\dots,T$ }
		\STATE Assign learning rate $\alpha_t$.
		\STATE Randomly choose either UPDATE(A) or UPDATE(B) with probability 0.5, respectively.
		\FOR{each $(s,a)$}
		\STATE observe $ s'\sim P(\cdot|s,a)$, and  sample $R_t(s,a,s')$.
		\IF{UPDATE(A)} 
		\STATE Obtain $a^* = \arg\max_{a'}Q^A_t(s',a')$
		\STATE $Q^A_{t+1}(s,a) = Q^A_t(s,a) + \alpha_t(s,a) (R_t(s,a,s') + \gamma Q^B_t(s',a^*) - Q^A_t(s,a))$
		\ELSIF{UPDATE(B)} 
		\STATE Obtain $b^* = \arg\max_{b'}Q^B_t(s',b')$
		\STATE $Q^B_{t+1}(s,a) = Q^B_t(s,a) + \alpha_t(s,a) (R_t(s,a,s') + \gamma Q^A_t(s',b^*) - Q^B_t(s,a))$
		\ENDIF
		\ENDFOR
		\ENDFOR
		\STATE 	{\bf Output:} $Q^A_T$ (or $Q^B_T$).
 	\end{algorithmic}
\end{algorithm}

Double Q-learning maintains two Q-estimators (i.e., Q-tables): $Q^A$ and $Q^B$. At each iteration of Algorithm \ref{alg:doubleQ}, one Q-table is randomly chosen to be updated. Then this chosen Q-table generates a greedy optimal action, and the other Q-table is used for estimating the corresponding Bellman operator for updating the chosen table. Specifically, if $Q^A$ is chosen to be updated, we use $Q^A$ to obtain the optimal action $a^*$ and then estimate the corresponding Bellman operator using $Q^B$. As shown in ~\citet{hasselt2010double}, $\mathbb{E}[Q^B(s',a^*)]$ is likely smaller than $\max_a \mathbb{E}[Q^A(s',a)]$, where the expectation is taken over the randomness of the reward for the same state-action pair. In this way, such a two-estimator framework of double Q-learning can effectively reduce the overestimation. 

{\bf Synchronous and asynchronous double Q-learning:} In this paper, we study the finite-time convergence rate of double Q-learning in two different settings: synchronous and asynchronous implementations. For synchronous double Q-learning (as shown in Algorithm \ref{alg:doubleQ}), all the state-action pairs of the chosen Q-estimator are visited simultaneously at each iteration. For the asynchronous case, only one state-action pair is updated in the chosen Q-table. Specifically, in the latter case, we sample a trajectory $\{(s_t,a_t,R_t,i_t)\}_{t=0}^\infty$ under a certain exploration strategy, where $i_t\in\{A,B\}$ denotes the index of the chosen Q-table at time $t$. Then the two Q-tables are updated based on the following rule:
\begin{multline*}
    Q^i_{t+1}(s,a) \\
    = \left\{\begin{aligned}
    & Q^i_t(s,a), \qquad (s,a) \neq (s_t,a_t) \text{ or } i\neq i_t;\\
    & (1-\alpha_t(s,a)) Q^i_t(s,a) + \alpha_t(s,a) \Big( R_t(s,a,s') + \gamma Q^{i^c}_t(s',\underset{a'\in U(s')}{\arg\max}Q^i_t(s',a') \Big), \  \text{otherwise},
    \end{aligned}
    \right.
\end{multline*}
where $i^c = \{A,B\}\setminus i$.

We next provide the boundedness property of the Q-estimators and the errors in the following lemma, which is typically necessary for the finite-time analysis. 
\begin{lemma}\label{lem:uniformBound}
For either synchronous or asynchronous double Q-learning, let $Q^i_t(s,a)$ be the value of either Q table corresponding to a state-action pair $(s,a)$ at iteration $t$. Suppose $\norm{Q_0^i}\leq \frac{R_{\max}}{1-\gamma}$. Then we have $\norm{Q^i_t}\leq \frac{R_{\max}}{1-\gamma}$ and $\norm{Q^i_t-Q^*}\leq V_{\max}$ for all $t\geq0$, where $V_{\max} := \frac{2R_{\max}}{1-\gamma}$.
\end{lemma}
Lemma \ref{lem:uniformBound} can be proved by induction arguments using the triangle inequality and the uniform boundedness of the reward function, which is seen in \Cref{sec:proofOfLemma1}.


\section{Main results}

We present our finite-time analysis for the synchronous and asynchronous double Q-learning in this section, followed by a sketch of the proof for the synchronous case which captures our main techniques. The detailed proofs of all the results are provided in the Supplementary Materials.

\subsection{Synchronous double Q-learning}
Since the update of the two Q-estimators is symmetric, we can characterize the convergence rate of either Q-estimator, e.g., $Q^A$, to the global optimum $Q^*$. To this end, we first derive two important properties of double Q-learning that are crucial to our finite-time convergence analysis.

The first property captures the stochastic error $\norm{Q^B_t - Q^A_t}$ between the two Q-estimators. Since double Q-learning updates alternatingly between these two estimators, such an error process must decay to zero in order for double Q-learning to converge. Furthermore, how fast such an error converges determines the overall convergence rate of double Q-learning. The following proposition (which is an informal restatement of~\Cref{lem:Gq} in~\Cref{subsec:PartI}) shows that such an error process can be \textit{block-wisely} bounded by an exponentially decreasing sequence $G_q = (1-\xi)^q V_{\max}$ for $q=0,1,2,\dots,$ and some $\xi\in(0,1)$. 
Conceptually, as illustrated in \Cref{fig:DkGk}, such an error process is upper-bounded by the blue-colored piece-wise linear curve.


\begin{repproposition}{lem:Gq}({\bf\em Informal})
 Consider synchronous double Q-learning under a polynomial learning rate $\alpha_t = \frac{1}{t^\omega}$ with $\omega\in(0,1)$. We divide the time horizon into blocks $[\hat{\tau}_{q},\hat{\tau}_{q+1})$ for $q\geq0$, where $\hat{\tau}_0=0$ and $\hat{\tau}_{q+1} = \hat{\tau}_q + c_1\hat{\tau}_q^\omega$ with some $c_1>0$. Fix $\hat{\epsilon}>0$. Then for any $n$ such that $G_n \geq \hat{\epsilon}$ and under certain conditions on $\hat{\tau}_1$ (see~\Cref{subsec:PartI}), we have
\begin{equation*}
    \mP\left[ \forall q\in [0,n], \forall t\in[\hat{\tau}_{q+1},\hat{\tau}_{q+2}), \norm{Q^B_t - Q^A_t}\leq G_{q+1} \right]\geq 1- c_2 n \exp\left( -\frac{c_3\hat{\tau}_1^{\omega}\hat{\epsilon}^2}{V_{\max}^2} \right),
\end{equation*}
where the positive constants $c_2$ and $c_3$ are specified in~\Cref{subsec:PartI}.
\end{repproposition}
\begin{figure}[h]
    \centering
    \includegraphics[width=0.70\textwidth]{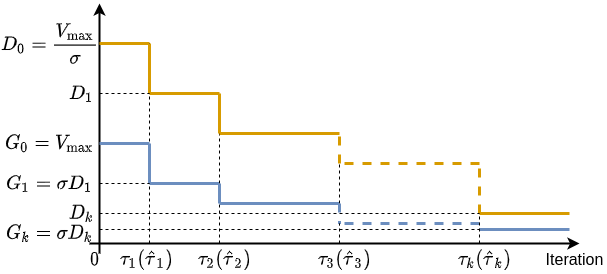}
    \caption{Illustration of sequence $\{G_k\}_{k\geq0}$ as a block-wise upper bound on $\norm{Q^B_t-Q^A_t}$, and sequence $\{D_k\}_{k\geq0}$ as a block-wise upper bound on $\norm{Q^A_t-Q^*}$ conditioned on the first upper bound event.}
    \label{fig:DkGk}
\end{figure}
\Cref{lem:Gq} implies that the two Q-estimators approach each other asymptotically, but does not necessarily imply that they converge to the optimal action-value function $Q^*$. Then the next proposition (which is an informal restatement of \Cref{lem:conditionalBound} in~\Cref{subsec:PartII}) shows that as long as the high probability event in~\Cref{lem:Gq} holds, the error process  $\norm{Q^A_t-Q^*}$ between either Q-estimator (say $Q^A$) and the optimal Q-function can be \textit{block-wisely} bounded by an exponentially decreasing sequence $D_k = (1-\beta)^k\frac{V_{\max}}{\sigma}$ for $k=0,1,2,\dots,$ and $\beta\in(0,1)$. Conceptually, as illustrated in \Cref{fig:DkGk}, such an error process is upper-bounded by the yellow-colored piece-wise linear curve.


\begin{repproposition}{lem:conditionalBound}
({\bf\em Informal}) Consider synchronous double Q-learning using a polynomial learning rate $\alpha_t = \frac{1}{t^\omega}$ with $\omega\in(0,1)$. We divide the time horizon into blocks $[\tau_{k},\!\tau_{k+1})$ for $k\geq0$, where ${\tau_{0}}=0$ and ${\tau}_{k+1} = {\tau}_k + c_4{\tau}_k^\omega$ with some $c_4>0$. Fix $\tilde{\epsilon}>0$. Then for any $m$ such that $D_m \geq \tilde{\epsilon}$ and under certain conditions on $\tau_1$ (see~\Cref{subsec:PartII}), we have
\begin{equation*}
    \mP\left[ \forall k\in [0,m], \forall t\in[\tau_{k+1},\tau_{k+2}), \norm{Q^A_t - Q^*}\leq D_{k+1} |E,F \right] \geq 1 - c_5 m \exp\left( -\frac{c_6\tau_1^{\omega}\tilde{\epsilon}^2}{V_{\max}^2} \right),
\end{equation*}
where $E$ and $F$ denote certain events defined in~\eqref{eq:eventA} and \eqref{eq:eventB} in~\Cref{subsec:PartII}, and the positive constants $c_4,c_5$, and $c_6$ are specified~\Cref{subsec:PartII}.
\end{repproposition}


As illustrated in~\Cref{fig:DkGk}, the two block sequences $\{\hat{\tau}_q\}_{q\geq0}$ in~\Cref{lem:Gq} and $\{{\tau_q}\}_{q\geq0}$ in~\Cref{lem:conditionalBound} can be chosen to coincide with each other. Then combining the above two properties followed by further mathematical arguments yields the following main theorem that characterizes the convergence rate of double Q-learning. We will provide a proof sketch for~\Cref{thm:syncDQ} in~\Cref{subsec:proofOutline}, which explains the main steps to obtain the supporting properties of~\Cref{lem:Gq} and~\ref{lem:conditionalBound} and how they further yield the main theorem.
\begin{theorem}\label{thm:syncDQ}
Fix $\epsilon>0$ and $\gamma\in(1/3,1)$. Consider synchronous double Q-learning using a polynomial learning rate $\alpha_t = \frac{1}{t^\omega}$ with $\omega\in(0,1)$. Let $Q^A_T(s,a)$ be the value of $Q^A$ for a state-action pair $(s,a)$ at time $T$. Then we have $\mP(\norm{Q^A_T - Q^*} \leq \epsilon)\geq 1-\delta$, given that
\begin{equation} \label{thm1T}
    T=\Omega\left( \left( \frac{V_{\max}^2}{(1-\gamma)^4\epsilon^2}\ln \frac{|\mcs||\mca|V_{\max}^2}{(1-\gamma)^5\epsilon^2\delta} \right)^{\frac{1}{\omega}} + \left(\frac{1}{1-\gamma} \ln\frac{ V_{\max}}{(1-\gamma)\epsilon} \right)^{\frac{1}{1-\omega}} \right),
\end{equation}
where $V_{\max} = \frac{2R_{\max}}{1-\gamma}$.
\end{theorem}
\Cref{thm:syncDQ} provides the finite-time convergence guarantee in high probability sense for synchronous double Q-learning. Specifically, double Q-learning attains an $\epsilon$-accurate optimal Q-function with high probability with at most $\Omega\left( \left( \frac{1}{(1-\gamma)^6\epsilon^2}\ln \frac{1}{(1-\gamma)^7\epsilon^2} \right)^{\frac{1}{\omega}} + \left(\frac{1}{1-\gamma} \ln\frac{ 1}{(1-\gamma)^2\epsilon} \right)^{\frac{1}{1-\omega}} \right)$ iterations. Such a result can be further understood by considering the following two regimes. In the high accuracy regime, in which $\epsilon \ll 1-\gamma$, the dependence on $\epsilon$ dominates, and the time complexity is given by $\Omega\left( \left( \frac{1}{\epsilon^2}\ln \frac{1}{\epsilon^2} \right)^{\frac{1}{\omega}} + \left( \ln\frac{ 1}{\epsilon} \right)^{\frac{1}{1-\omega}} \right)$, which is optimized as $\omega$ approaches to 1. In the low accuracy regime, in which $\epsilon \gg 1-\gamma$, the dependence on $\frac{1}{1-\gamma}$ dominates, and 
the time complexity can be optimized at $\omega=\frac{6}{7}$, which yields $T=\tilde{\Omega} \left( \frac{1}{(1-\gamma)^7\epsilon^{7/3}} + \frac{1}{(1-\gamma)^7}  \right)=\tilde{\Omega} \left( \frac{1}{(1-\gamma)^7\epsilon^{7/3}} \right)$. 



Furthermore, \Cref{thm:syncDQ} corroborates the design effectiveness of double Q-learning, which overcomes the overestimation issue and hence achieves better accuracy by making less aggressive progress in each update. Specifically, comparison of \Cref{thm:syncDQ} with the time complexity bounds of vanilla synchronous Q-learning under a polynomial learning rate in \citet{even2003learning} and \citet{wainwright2019stochastic} indicates that in the high accuracy regime, double Q-learning achieves the same convergence rate as vanilla Q-learning in terms of the order-level dependence on $\epsilon$. Clearly, the design of double Q-learning for high accuracy dominates the performance. In the low-accuracy regime (which is not what double Q-learning is designed for), double Q-learning achieves a slightly weaker convergence rate than vanilla Q-learning in \citet{even2003learning,wainwright2019stochastic} in terms of the dependence on $1-\gamma$, because its nature of less aggressive progress dominates the performance. 

\subsection{Asynchronous Double Q-learning}

In this subsection, we study the asynchronous double Q-learning and provide its finite-time convergence result.

Differently from synchronous double Q-learning, in which all state-action pairs are visited for each update of the chosen Q-estimator, asynchronous double Q-learning visits only one state-action pair for each update of the chosen Q-estimator. Therefore, we make the following standard assumption on the exploration strategy~\citep{even2003learning}:

\begin{assumption}\label{asp:covering}
(Covering number) There exists a covering number $L$, such that in consecutive $L$ updates of either $Q^A$ or $Q^B$ estimator, all the state-action pairs of the chosen Q-estimator are visited at least once. 
\end{assumption}

The above conditions on the exploration are usually necessary for the finite-time analysis of asynchronous Q-learning. The same assumption has been taken in \citet{even2003learning}. \citet{qu2020finite} proposed a mixing time condition which is in the same spirit.  

\Cref{asp:covering} essentially requires the sampling strategy to have good visitation coverage over all state-action pairs. Specifically,
\Cref{asp:covering} guarantees that consecutive $L$ updates of $Q^A$ visit each state-action pair of $Q^A$ at least once, and the same holds for $Q^B$. Since $2L$ iterations of asynchronous double Q-learning must make at least $L$ updates for either $Q^A$ or $Q^B$, \Cref{asp:covering} further implies that any state-action pair $(s,a)$ must be visited at least once during $2L$ iterations of the algorithm. In fact, our analysis allows certain relaxation of~\Cref{asp:covering} by only requiring each state-action pair to be visited during an interval with a certain probability. In such a case, we can also derive a finite-time bound by additionally dealing with a conditional probability.

Next, we provide the finite-time result for asynchronous double Q-learning in the following theorem.

\begin{theorem}\label{thm:asyncDQ}
Fix $\epsilon>0,\gamma\in(1/3,1)$. Consider asynchronous double Q-learning under a polynomial learning rate $\alpha_t = \frac{1}{t^\omega}$ with $\omega\in(0,1)$. Suppose Assumption \ref{asp:covering} holds. Let $Q^A_T(s,a)$ be the value of $Q^A$ for a state-action pair $(s,a)$ at time $T$. Then we have $\mP(\norm{Q^A_T - Q^*} \leq \epsilon)\geq 1-\delta$, given that
\begin{equation} \label{thm2T}
    T=\Omega\left( \left(\frac{ L^4V_{\max}^2}{(1-\gamma)^4\epsilon^2}\ln \frac{|\mcs||\mca|L^4V_{\max}^2 }{(1-\gamma)^5\epsilon^2\delta} \right)^{\frac{1}{\omega}} + \left(\frac{L^2}{1-\gamma} \ln\frac{\gamma V_{\max}}{(1-\gamma)\epsilon} \right)^{\frac{1}{1-\omega}} \right).
\end{equation}
\end{theorem}

Comparison of~\Cref{thm:syncDQ} and~\ref{thm:asyncDQ} indicates that the finite-time result of asynchronous double Q-learning matches that of synchronous double Q-learning in the order dependence on $\frac{1}{1-\gamma}$ and $\frac{1}{\epsilon}$. The difference lies in the extra dependence on the covering time $L$ in \Cref{thm:asyncDQ}. Since synchronous double Q-learning visits all state-action pairs (i.e., takes $|\mcs||\mca|$ sample updates) at each iteration, whereas asynchronous double Q-learning visits only one state-action pair (i.e., takes only one sample update) at each iteration, a more reasonable comparison between the two should be in terms of the overall sample complexity. In this sense, synchronous and asynchronous double Q-learning algorithms have the sample complexities of $|\mcs||\mca| T$ (where $T$ is given in~\eqref{thm1T}) and $T$ (where $T$ is given in~\eqref{thm2T}), respectively. Since in general $L\gg |\mcs||\mca|$, synchronous double-Q is more efficient than asynchronous double-Q in terms of the overall sampling complexity.


\subsection{Proof Sketch of Theorem \ref{thm:syncDQ}}\label{subsec:proofOutline}

In this subsection, we provide an outline of the technical proof of Theorem \ref{thm:syncDQ} and summarize the key ideas behind the proof. The detailed proof can be found in~\Cref{sec:proofThm1}.


Our goal is to study the finite-time convergence of the error $\norm{Q^A_t-Q^*}$ between one Q-estimator and the optimal Q-function (this is without the loss of generality due to the symmetry of the two estimators). To this end, our proof includes: (a) Part I which analyzes the stochastic error propagation between the two Q-estimators $\norm{Q^B_t - Q^A_t}$; (b) Part II which analyzes the error dynamics between one Q-estimator and the optimum $\norm{Q^A_t-Q^*}$ conditioned on the error event in Part I; and (c) Part III which bounds the unconditional error $\norm{Q^A_t-Q^*}$. We describe each of the three parts in more details below.



\textbf{Part I: Bounding $\norm{Q^B_t - Q^A_t}$ (see \Cref{lem:Gq}).} The main idea is to upper bound $\norm{Q^B_t - Q^A_t}$ by a decreasing sequence $\{G_q\}_{q\geq0}$ block-wisely with high probability, where each block $q$ (with $q\geq0$) is defined by $t\in[\hat{\tau}_q, \hat{\tau}_{q+1})$. The proof consists of the following four steps.


\textit{Step 1 (see \Cref{lem:uBAdyn})}: We characterize the dynamics of $u_t^{BA}(s,a):=Q^B(s,a) - Q^A(s,a)$ as an SA algorithm as follows: 
\begin{equation*}
    u_{t+1}^{BA}(s,a) = (1-\alpha_t)u_t^{BA}(s,a) + \alpha_t (h_t(s,a) + z_t(s,a)),
\end{equation*}
where $h_t$ is a contractive mapping of $u_t^{BA}$, and $z_t$ is a martingale difference sequence. 

\textit{Step 2 (see \Cref{lem:uBAsanwich})}:
We derive lower and upper bounds on $u_t^{BA}$ via two sequences $X_{t;\hat{\tau}_q}$ and $Z_{t;\hat{\tau}_q}$ as follows:
\begin{equation*}
    -X_{t;\hat{\tau}_q}(s,a) + Z_{t;\hat{\tau}_q}(s,a) \leq u_t^{BA}(s,a) \leq X_{t;\hat{\tau}_q}(s,a) + Z_{t;\hat{\tau}_q}(s,a),
\end{equation*}
for any $t\geq \hat{\tau}_q$, state-action pair $(s,a)\in\mathcal{S}\times\mathcal{A}$, and $q\geq0$, where $X_{t;\hat{\tau}_q}$ is deterministic and driven by $G_q$, and $Z_{t;\hat{\tau}_q}$ is stochastic and driven by the martingale difference sequence $z_t$. 

\textit{Step 3 (see \Cref{lem:Xt} and \Cref{lem:ZlDiff})}:
We block-wisely bound $u_t^{BA}(s,a)$ using the induction arguments. Namely, we prove $\norm{u_t^{BA}}\leq G_q$ for $t\in[\hat{\tau_q},\hat{\tau}_{q+1})$ holds for all $q\geq0$. By induction, we first observe for $q=0$, $\norm{u_t^{BA}}\leq G_0$ holds. Given any state-action pair $(s,a)$, we assume that $\norm{u_t^{BA}(s,a)}\leq G_q$ holds for $t\in[\hat{\tau}_q, \hat{\tau}_{q+1})$. Then we show $\norm{u_t^{BA}(s,a)}\leq G_{q+1}$ holds for $t\in[\hat{\tau}_{q+1}, \hat{\tau}_{q+2})$, which follows by bounding $X_{t;\hat{\tau}_q}$ and $Z_{t;\hat{\tau}_q}$ separately in \Cref{lem:Xt} and \Cref{lem:ZlDiff}, respectively.

\textit{Step 4 (see \Cref{subsec:proofProp1})} :
We apply union bound (\Cref{lem:unionBound}) to obtain the block-wise bound for all state-action pairs and all blocks.


\textbf{Part II: Conditionally bounding $\norm{Q^A_t - Q^*}$ (see~\Cref{lem:conditionalBound})}.
We upper bound $\norm{Q^A_t - Q^*} $ by a decreasing sequence $\{D_k\}_{k\geq 0}$ block-wisely conditioned on the following two events:
\begin{itemize}
    \item []Event $E$: $\norm{u^{BA}_t}$ is upper bounded properly (see \eqref{eq:eventA} in~\Cref{subsec:PartII}), and
    \item []Event $F$: there are sufficient updates of $Q^A_t$ in each block (see~\eqref{eq:eventB} in~\Cref{subsec:PartII}). 
\end{itemize}

The proof of~\Cref{lem:conditionalBound} consists of the following four steps.

\textit{Step 1 (see~\Cref{lem:couple})}: We design a special relationship (illustrated in~\Cref{fig:DkGk}) between the block-wise bounds $\{G_q\}_{q\geq 0}$ and $\{D_k\}_{k\geq 0}$ and their block separations.

\textit{Step 2 (see~\Cref{lem:residualDynamics})}: We characterize the dynamics of the iteration residual $r_{t}(s,a):=Q^A_t(s,a) - Q^*(s,a)$ as an SA algorithm as follows:
when $Q^A$ is chosen to be updated at iteration $t$,
\begin{equation*}
    r_{t+1}(s,a) \!=\! (1\!-\!\alpha_{t}) r_{t}(s,a) \!+\! \alpha_{t} (\mcT Q_{t}^A(s,a)\!-\!Q^*(s,a)) \!+\! \alpha_{t} w_{t}(s,a) \!+\! \alpha_{t}\gamma u_{t}^{BA}(s',a^*),
\end{equation*}
where $w_{t}(s,a)$ is the error between the Bellman operator and the sample-based empirical estimator, and is thus a martingale difference sequence, and $u_{t}^{BA}$ has been defined in Part I.

\textit{Step 3 (see~\Cref{lem:rtSandwich})}:
We provide upper and lower bounds on $r_t$ via two sequences $Y_{t;\tau_k}$ and $W_{t;\tau_k}$ as follows:
\begin{equation*}
    -Y_{t;\tau_k}(s,a) + W_{t;\tau_k}(s,a) \leq r_t(s,a) \leq Y_{t;\tau_k}(s,a) + W_{t;\tau_k}(s,a),
\end{equation*}
for all $t\geq \tau_k$, all state-action pairs $(s,a)\in\mathcal{S}\times\mathcal{A}$, and all $q\geq0$, where $Y_{t;\tau_k}$ is deterministic and driven by $D_k$, and $W_{t;\tau_k}$ is stochastic and driven by the martingale difference sequence $w_t$. In particular, if $Q^A_t$ is not updated at some iteration, then the sequences $Y_{t;\tau_k}$ and $W_{t;\tau_k}$ assume the same values from the previous iteration.

\textit{Step 4 (see \Cref{lem:Yt}, \Cref{lem:WlDiff} and \Cref{subsec:proofProp2})}:
Similarly to Steps 3 and 4 in Part I, we conditionally bound $\norm{r_t}\leq D_k$ for $t\in [\tau_{k}, \tau_{k+1})$ and $k\geq0$ via bounding $Y_{t;\tau_k}$ and $W_{t;\tau_k}$ and further taking the union bound.

\textbf{Part III: Bounding $\norm{Q^A_t - Q^*}$(see \Cref{subsec:proofThm1}).} We combine the results in the first two parts, and provide high probability bound on $\norm{r_t}$ with further probabilistic arguments, which exploit the high probability bounds on $\mP(E)$ in~\Cref{lem:Gq} and $\mP(F)$ in~\Cref{lem:halfQA}.


\section{Conclusion}

In this paper, we provide the first finite-time results for double Q-learning, which characterize how fast double Q-learning converges under both synchronous and asynchronous implementations.
For the synchronous case, we show that it achieves an $\epsilon$-accurate optimal Q-function with at least the probability of $1-\delta$ by taking $\Omega\left( \left( \frac{1}{(1-\gamma)^6\epsilon^2}\ln \frac{|\mcs||\mca|}{(1-\gamma)^7\epsilon^2\delta} \right)^{\frac{1}{\omega}} + \left(\frac{1}{1-\gamma} \ln\frac{ 1}{(1-\gamma)^2\epsilon} \right)^{\frac{1}{1-\omega}} \right)$ iterations.
Similar scaling order on $\frac{1}{1-\gamma}$ and $\frac{1}{\epsilon}$ also applies for asynchronous double Q-learning but with extra dependence on the covering number. We develop new techniques to bound the error between two correlated stochastic processes, which can be of independent interest.

\section*{Acknowledgements}

The work was supported in part by the U.S. National Science Foundation under the grant CCF-1761506 and the startup fund of the Southern University of Science and Technology (SUSTech), China.

\section*{Broader Impact}

Reinforcement learning has achieved great success in areas such as robotics and game playing, and thus has aroused broad interests and more potential real-world applications. Double Q-learning is a commonly used technique in deep reinforcement learning to improve the implementation stability and speed of deep Q-learning.
In this paper, we provided the fundamental analysis on the convergence rate for double Q-learning, which theoretically justified the empirical success of double Q-learning in practice. Such a theory also provides practitioners desirable performance guarantee to further develop such a technique into various transferable technologies.





\bibliography{doubleQ}
\bibliographystyle{apalike}

\newpage

\appendix

\include{Supplementary}

\end{document}

%% file: Supplementary.tex
\textbf{\huge Supplementary Materials}

\section{Proof of Lemma \ref{lem:uniformBound}} \label{sec:proofOfLemma1}

We prove \Cref{lem:uniformBound} by induction.

First, it is easy to guarantee that the initial case is satisfied, i.e., $\norm{Q^A_1}\leq \frac{R_{\max}}{1-\gamma} = \frac{V_{\max}}{2}, \norm{Q^B_1}\leq \frac{V_{\max}}{2}$. (In practice we usually initialize the algorithm as $Q^A_1=Q^B_1=0$). Next, we assume that $\norm{Q^A_t}\leq \frac{V_{\max}}{2}, \norm{Q^B_t}\leq \frac{V_{\max}}{2}$. It remains to show that such conditions still hold for $t+1$.

Observe that
\begin{align*}
    \norm{Q^A_{t+1}(s,a)} &= \norm{ (1-\alpha_t)Q^A_{t}(s,a) + \alpha_t\left( R_t+\gamma Q^{B}_t(s',\underset{a'\in U(s')}{\arg\max}Q^A_t(s',a') \right) } \\
    &\leq (1-\alpha_t)\norm{Q^A_{t}} + \alpha_t\norm{R_t} + \alpha_t\gamma\norm{Q^{B}_t}\\
    &\leq (1-\alpha_t)\frac{R_{\max}}{1-\gamma} + \alpha_t R_{\max} + \frac{\alpha_t\gamma R_{\max}}{1-\gamma}\\
    &= \frac{R_{\max}}{1-\gamma} = \frac{V_{\max}}{2}.
\end{align*}

Similarly, we can have $\norm{Q^B_{t+1}(s,a)}\leq \frac{V_{\max}}{2}$. Thus we complete the proof.

\section{Proof of Theorem \ref{thm:syncDQ}} \label{sec:proofThm1}

In this appendix, we will provide a detailed proof of~\Cref{thm:syncDQ}. Our proof includes: (a) Part I which analyzes the stochastic error propagation between the two Q-estimators $\norm{Q^B_t  - Q^A_t }$; (b) Part II which analyzes the error dynamics between one Q-estimator and the optimum $\norm{Q^A_t -Q^* }$ conditioned on the error event in Part I; and (c) Part III which bounds the unconditional error $\norm{Q^A_t -Q^* }$. We describe each of the three parts in more details below.

\subsection{Part I: Bounding $\norm{Q^B_t  - Q^A_t }$} \label{subsec:PartI}
The main idea is to upper bound $\norm{Q^B_t  - Q^A_t }$ by a decreasing sequence $\{G_q\}_{q\geq0}$ block-wisely with high probability, where each block or epoch $q$ (with $q\geq0$) is defined by $t\in[\hat{\tau}_q, \hat{\tau}_{q+1})$.
\begin{proposition}
\label{lem:Gq}
Fix $\epsilon>0, \kappa\in(0,1), \sigma\in(0,1)$ and $\Delta\in(0, e-2)$. Consider synchronous double Q-learning using a polynomial learning rate $\alpha_t = \frac{1}{t^\omega}$ with $\omega\in(0,1)$. Let $G_q = (1-\xi)^q G_0$ with $G_0 = V_{\max}$ and $\xi=\frac{1-\gamma}{4}$. Let $\hat{\tau}_{q+1} = \hat{\tau}_q + \frac{2c}{\kappa}\hat{\tau}_q^\omega$ for $q \geq 1$ with $c\geq \frac{\ln(2+\Delta)+1/\hat{\tau}_1^\omega}{1-\ln(2+\Delta)-1/\hat{\tau}_1^\omega}$ and $\hat{\tau}_1$ as the finishing time of the first epoch satisfying 
\begin{equation*}
    \hat{\tau}_1\geq  \max\left\{\left(\frac{1}{1-\ln(2+\Delta)}\right)^{\frac{1}{\omega}}, \left( \frac{128c(c+\kappa)V_{\max}^2}{\kappa^2\left(\frac{\Delta}{2+\Delta}\right)^2\sigma^2\xi^2\epsilon^2 }\ln\left(\frac{64c(c+\kappa)V_{\max}^2}{\kappa^2\left(\frac{\Delta}{2+\Delta}\right)^2\sigma^2\xi^2\epsilon^2 }\right) \right)^{\frac{1}{\omega}} \right\}.
\end{equation*}
Then for any $n$ such that $G_n\geq\sigma\epsilon$, we have
\begin{align*}
    &\mP\left[ \forall q\in [0,n], \forall t\in[\hat{\tau}_{q+1},\hat{\tau}_{q+2}), \norm{Q^B_t  - Q^A_t }\leq G_{q+1} \right]\\
        &\quad\geq 1- \frac{4c(n+1)}{\kappa}\left(1+\frac{2c}{\kappa}\right)|\mcs||\mca| \exp\left( -\frac{\kappa^2\left( \frac{\Delta}{2+\Delta}\right)^2\xi^2 \sigma^2\epsilon^2\hat{\tau}_1^\omega}{64c(c+\kappa)V_{\max}^2} \right).
\end{align*}
\end{proposition}

The proof of Proposition \ref{lem:Gq} consists of the following four steps.

\subsubsection{Step 1: Characterizing the dynamics of $Q_t^B(s,a) - Q_t^A(s,a)$ } 
We first characterize the dynamics of $u_t^{BA}(s,a):=Q_t^B(s,a) - Q_t^A(s,a)$ as a stochastic approximation (SA) algorithm in this step.
\begin{lemma}\label{lem:uBAdyn}
Consider double Q-learning in Algorithm \ref{alg:doubleQ}. Then we have
\begin{equation*}
    u_{t+1}^{BA}(s,a) = (1-\alpha_t)u_t^{BA}(s,a) + \alpha_t F_t(s,a),
\end{equation*}
where
\begin{equation*}
    F_t(s,a) = \left\{\begin{aligned}
    &Q^B_t(s,a) - R_t - \gamma Q^B_t(s_{t+1},a^*),\quad\text{w.p. 1/2}  \\
    &R_t + \gamma Q^A_t(s_{t+1},b^*) - Q^A_t(s,a),\quad\text{w.p. 1/2}.
    \end{aligned}
    \right.
\end{equation*}
In addition, $F_t$ satisfies
\begin{equation*}
    \norm{\mE[F_t|\mcF_t]} \leq \frac{1+\gamma}{2} \norm{u_t^{BA} }.
\end{equation*}
\end{lemma}
\begin{proof}
Algorithm \ref{alg:doubleQ} indicates that at each time, either $Q^A$ or $Q^B$ is updated with equal probability. When updating $Q^A$ at time $t$, for each $(s,a)$ we have
\begin{align*}
    u_{t+1}^{BA}(s,a) &= Q^B_{t+1}(s,a) - Q^A_{t+1}(s,a)\\
    &= Q^B_t(s,a) - (Q^A_t(s,a) + \alpha_t (R_t + \gamma Q^B_t(s_{t+1},a^*) - Q^A_t(s,a)))\\
    &= (1-\alpha_t) Q^B_t(s,a) - ( (1-\alpha_t)Q^A_t(s,a) + \alpha_t (R_t + \gamma Q^B_t(s_{t+1},a^*) - Q^B_t(s,a)) )\\
    &= (1-\alpha_t) u_t^{BA}(s,a) + \alpha_t(Q^B_t(s,a) - R_t - \gamma Q^B_t(s_{t+1},a^*) ).
\end{align*}
Similarly, when updating $Q^B$, we have
\begin{align*}
    u_{t+1}^{BA}(s,a) &= Q^B_{t+1}(s,a) - Q^A_{t+1}(s,a)\\
    &= (Q^B_t(s,a) + \alpha_t (R_t + \gamma Q^A_t(s_{t+1},b^*) - Q^B_t(s,a))) - Q^A_t(s,a)\\
    &= (1-\alpha_t) Q^B_t(s,a) + ( \alpha_t (R_t + \gamma Q^A_t(s_{t+1},b^*) - Q^A_t(s,a)) - (1-\alpha_t)Q^A_t(s,a) )\\
    &= (1-\alpha_t) u_t^{BA}(s,a) + \alpha_t(R_t + \gamma Q^A_t(s_{t+1},b^*) - Q^A_t(s,a)).
\end{align*}
Therefore, we can rewrite the dynamics of $u_t^{BA}$ as $u_{t+1}^{BA}(s,a) = (1-\alpha_t)u_t^{BA}(s,a) + \alpha_t F_t(s,a)$, where
\begin{equation*}
    F_t(s,a) = \left\{\begin{aligned}
    &Q^B_t(s,a) - R_t - \gamma Q^B_t(s_{t+1},a^*),\quad\text{w.p. 1/2}  \\
    &R_t + \gamma Q^A_t(s_{t+1},b^*) - Q^A_t(s,a),\quad\text{w.p. 1/2}.
    \end{aligned}
    \right.
\end{equation*}
Thus, we have
\begin{align}
    \mE&[F_t(s,a)|\mcF_t] \nonumber\\  
    &= \frac{1}{2}\left( Q^B_t(sa)\! - \!\underset{s_{t+1}}{\mE}[R_{sa}^{s'} \!- \!\gamma  Q^B_t(s_{t+1},a^*)] \right)\! +\! \frac{1}{2}\left(\! \underset{s_{t+1}}{\mE}[R_{s,a}^{s'}\! +\! \gamma  Q^A_t(s_{t+1},b^*)\!]\! -\! Q^A_t(s,a) \right)\nonumber\\
    &= \frac{1}{2} (Q^B_t(s,a) - Q^A_t(s,a)) + \frac{\gamma}{2}  \underset{s_{t+1}}{\mE}\left[ Q^A_t(s_{t+1},b^*) - Q^B_t(s_{t+1},a^*) \right]\nonumber\\
    &= \frac{1}{2} u_t^{BA}(s,a) + \frac{\gamma}{2}  \underset{s_{t+1}}{\mE}\left[ Q^A_t(s_{t+1},b^*) - Q^B_t(s_{t+1},a^*) \right]. \label{eq:pf1LemGq}
\end{align}
Next, we bound $\underset{s_{t+1}}{\mE}\left[ Q^A_t(s_{t+1},b^*) - Q^B_t(s_{t+1},a^*) \right]$. First, consider the case when $\underset{s_{t+1}}{\mE} Q^A_t(s_{t+1},b^*) \geq \underset{s_{t+1}}{\mE} Q^B_t(s_{t+1},a^*)$. Then we have
\begin{align*}
    \left| \underset{s_{t+1}}{\mE}\left[ Q^A_t(s_{t+1},b^*) - Q^B_t(s_{t+1},a^*) \right] \right| &= \underset{s_{t+1}}{\mE}\left[ Q^A_t(s_{t+1},b^*) - Q^B_t(s_{t+1},a^*) \right]\\
    &\overset{\text{(i)}}{\leq} \underset{s_{t+1}}{\mE}\left[ Q^A_t(s_{t+1},a^*) - Q^B_t(s_{t+1},a^*) \right]\\
    &\leq \norm{u_t^{BA}},
\end{align*}
where (i) follow from the definition of $a^*$ in Algorithm \ref{alg:doubleQ}. Similarly, if $\underset{s_{t+1}}{\mE} Q^A_t(s_{t+1},b^*) < \underset{s_{t+1}}{\mE} Q^B_t(s_{t+1},a^*)$, we have
\begin{align*}
    \left|  \underset{s_{t+1}}{\mE}\left[ Q^A_t(s_{t+1},b^*) - Q^B_t(s_{t+1},a^*) \right] \right| &= \underset{s_{t+1}}{\mE}\left[ Q^B_t(s_{t+1},a^*) - Q^A_t(s_{t+1},b^*) \right]\\
    &\overset{\text{(i)}}{\leq} \underset{s_{t+1}}{\mE}\left[ Q^B_t(s_{t+1},b^*) - Q^A_t(s_{t+1},b^*) \right]\\
    &\leq \norm{u_t^{BA}},
\end{align*}
where (i) follows from the definition of $b^*$.
Thus we can conclude that 
\begin{equation*}
    \left|  \underset{s_{t+1}}{\mE}\left[ Q^A_t(s_{t+1},b^*) - Q^B_t(s_{t+1},a^*) \right] \right| \leq \norm{u_t^{BA}}.
\end{equation*}
Then, we continue to bound \eqref{eq:pf1LemGq}, and obtain
\begin{align*}
    \left|\mE[F_t(s,a)|\mcF_t] \right| &= \left|\frac{1}{2} u_t^{BA}(s,a) + \frac{\gamma}{2}  \underset{s_{t+1}}{\mE}\left[ Q^A_t(s_{t+1},b^*) - Q^B_t(s_{t+1},a^*) \right]\right|\\
    &\leq \frac{1}{2}\norm{u_t^{BA}} + \frac{\gamma}{2}\left|\underset{s_{t+1}}{\mE}\left[ Q^A_t(s_{t+1},b^*) - Q^B_t(s_{t+1},a^*) \right]\right|\\
    &\leq \frac{1+\gamma}{2} \norm{u_t^{BA}},
\end{align*}
for all $(s,a)$ pairs. Hence, $\norm{\mE[F_t|\mcF_t]}\leq \frac{1+\gamma}{2} \norm{u_t^{BA}}$. 
\end{proof}

Applying~\Cref{lem:uBAdyn}, we write the dynamics of $u_t^{BA}(s,a)$ in the form of a classical SA algorithm driven by a martingale difference sequence as follows:
\begin{equation*}
    u_{t+1}^{BA}(s,a) = (1-\alpha_t)u_t^{BA}(s,a) + \alpha_t F_t(s,a) = (1-\alpha_t)u_t^{BA}(s,a) + \alpha_t (h_t(s,a) + z_t(s,a)),
\end{equation*}
where $h_t(s,a) = \mE[F_t(s,a)|\mcF_t]$ and $z_t(s,a) = F_t(s,a) - \mE[F_t|\mcF_t]$. Then, we obtain $\mE [z_t(s,a)|\mcF_t] = 0$ and $\norm{h_t} \leq \frac{1+\gamma}{2}\norm{u_t^{BA}}$ following from Lemma \ref{lem:uBAdyn}. We define $u^*(s,a) = 0$, and treat $h_t$ as an operator over $u^{BA}_t$. Then $h_t$ has a contraction property as:
\begin{equation}\label{eq:uContraction}
    \norm{h_t-u^*} \leq \gamma'\norm{u_t^{BA}-u^*},
\end{equation}
where $\gamma'=\frac{1+\gamma}{2}\in(0,1)$. Based on this SA formulation, we bound $u_t^{BA}(s,a)$ block-wisely in the next step. 

\subsubsection{Step 2: Constructing sandwich bounds on $u_t^{BA}$}
We derive lower and upper bounds on $u_t^{BA}$ via two sequences $X_{t;\hat{\tau}_q}$ and $Z_{t;\hat{\tau}_q}$ in the following lemma.

\begin{lemma}\label{lem:uBAsanwich}
Let $\hat{\tau}_q$ be such that $\norm{u_t^{BA}}\leq G_q$ for all $t\geq\hat{\tau}_q$. Define $Z_{t;\hat{\tau}_q}(s,a), X_{t;\hat{\tau}_q}(s,a)$ as
\begin{align*}
    Z_{t+1;\hat{\tau}_q}(s,a) &= (1-\alpha_t)Z_{t;\hat{\tau}_q}(s,a) + \alpha_t z_t(s,a), \quad \text{with } Z_{\hat{\tau}_q;\hat{\tau}_q}(s,a) = 0;\\
    X_{t+1;\hat{\tau}_q}(s,a) &= (1-\alpha_t)X_{t;\hat{\tau}_q}(s,a) + \alpha_t \gamma'G_q, \quad \text{with } X_{\hat{\tau}_q;\hat{\tau}_q}(s,a) = G_q,\gamma'=\frac{1+\gamma}{2}.
\end{align*}
Then for any $t\geq\hat{\tau}_q$ and state-action pair $(s,a)$, we have
\begin{equation*}
    -X_{t;\hat{\tau}_q}(s,a) + Z_{t;\hat{\tau}_q}(s,a) \leq u_t^{BA}(s,a) \leq X_{t;\hat{\tau}_q}(s,a) + Z_{t;\hat{\tau}_q}(s,a).
\end{equation*}
\end{lemma} 
\begin{proof}
We proceed the proof by induction. For the initial condition $t=\hat{\tau}_q$, $\norm{u_{\hat{\tau}_q}^{BA}}\leq G_q$ implies $-G_q\leq u_{\hat{\tau}_q}^{BA} \leq G_q$. We assume the sandwich bound holds for time $t$. It remains to check that the bound also holds for $t+1$.

At time $t+1$, we have
\begin{align*}
    u_{t+1}^{BA}(s,a) &= (1-\alpha_t)u_t^{BA}(s,a) + \alpha_t (h_t(s,a) + z_t(s,a))\\
    &\leq (1-\alpha_t)( X_{t;\hat{\tau}_q}(s,a) + Z_{t;\hat{\tau}_q}(s,a) ) + \alpha_t (h_t(s,a) + z_t(s,a))\\
    &\overset{\text{(i)}}{\leq} \left[(1-\alpha_t) X_{t;\hat{\tau}_q}(s,a) + \alpha_t \gamma'\norm{u_t^{BA}}\right] + \left[(1-\alpha_t) Z_{t;\hat{\tau}_q}(s,a) + \alpha_t z_t(s,a)\right]\\
    &\leq \left[(1-\alpha_t) X_{t;\hat{\tau}_q}(s,a) + \alpha_t \gamma'G_q\right] + \left[(1-\alpha_t) Z_{t;\hat{\tau}_q}(s,a) + \alpha_t z_t(s,a)\right]\\
    &= X_{t+1;\hat{\tau}_q}(s,a) + Z_{t+1;\hat{\tau}_q}(s,a),
\end{align*}
where (i) follows from Lemma \ref{lem:uBAdyn}. Similarly, we can bound the other direction as
\begin{align*}
    u_{t+1}^{BA}(s,a) &= (1-\alpha_t)u_t^{BA}(s,a) + \alpha_t (h_t(s,a) + z_t(s,a))\\
    &\geq (1-\alpha_t)( -X_{t;\hat{\tau}_q}(s,a) + Z_{t;\hat{\tau}_q}(s,a) ) + \alpha_t (h_t(s,a) + z_t(s,a))\\
    &\geq \left[-(1-\alpha_t) X_{t;\hat{\tau}_q}(s,a) - \alpha_t \gamma'\norm{u_t^{BA}}\right] + \left[(1-\alpha_t) Z_{t;\hat{\tau}_q}(s,a) + \alpha_t z_t(s,a)\right]\\
    &\geq \left[-(1-\alpha_t) X_{t;\hat{\tau}_q}(s,a) - \alpha_t \gamma'G_q\right] + \left[(1-\alpha_t) Z_{t;\hat{\tau}_q}(s,a) + \alpha_t z_t(s,a)\right]\\
    &= -X_{t+1;\hat{\tau}_q}(s,a) + Z_{t+1;\hat{\tau}_q}(s,a).
\end{align*}
\end{proof}

\subsubsection{Step 3: Bounding $X_{t;\hat{\tau}_q}$ and $Z_{t;\hat{\tau}_q}$ for block $q+1$}
We bound $X_{t;\hat{\tau}_q}$ and $Z_{t;\hat{\tau}_q}$ in \Cref{lem:Xt} and \Cref{lem:ZlDiff} below, respectively. 
Before that, we first introduce the following technical lemma which will be useful in the proof of~\Cref{lem:Xt}.

\begin{lemma}\label{lem:prodHelp}
Fix $\omega\in(0,1)$. Let $0 < t_1 < t_2$. Then we have
\begin{equation*}
    \prod_{i=t_1}^{t_2} \left( 1-\frac{1}{i^\omega} \right) \leq \exp\left( -\frac{t_2 - t_1}{t_2^\omega} \right).
\end{equation*}
\end{lemma}
\begin{proof}
Since $\ln(1-x)\leq -x$ for any $x\in (0,1)$, we have
\begin{equation*}
    \ln\left[ \prod_{i=t_1}^{t_2} \left( 1-\frac{1}{i^\omega} \right) \right] \leq -\sum_{i=t_1}^{t_2}i^{-\omega}\leq -\int_{t_1}^{t_2} t^{-\omega}dt = -\frac{t_2^{1-\omega} - t_1^{1-\omega}}{1-\omega}.
\end{equation*}
Thus, fix $\omega\in(0,1)$, let $0 < t_1 < t_2$, and then we have
\begin{equation*}
    \prod_{i=t_1}^{t_2} \left( 1-\frac{1}{i^\omega} \right) \leq \exp\left( -\frac{t_2^{1-\omega} - t_1^{1-\omega}}{1-\omega} \right).
\end{equation*}
Define $f(t) := t^{1-\omega}$. Observe that $f(t)$ is an increasing concave function. Then we have
\begin{align*}
    t_2^{1-\omega} - t_1^{1-\omega} &\geq f'(t_2)(t_2-t_1) = (1-\omega)t_2^{-\omega} (t_2 - t_1),
\end{align*}
which immediately indicates the result.
\end{proof}

We now derive a bound for $X_{t;\hat{\tau}_q}$.
\begin{lemma}\label{lem:Xt}
Fix $\kappa\in (0,1)$ and $\Delta\in(0, e-2)$. Let $\{G_q\}$ be defined in~\Cref{lem:Gq}. Consider synchronous double Q-learning using a polynomial learning rate $\alpha_t = \frac{1}{t^\omega}$ with $\omega\in(0,1)$. Suppose that $X_{t;\hat{\tau}_q}(s,a) \leq G_q$ for any $t \geq \hat{\tau}_q$. Then for any $t\in[\hat{\tau}_{q+1}, \hat{\tau}_{q+2})$, given $\hat{\tau}_{q+1} = \hat{\tau}_q + \frac{2c}{\kappa}\hat{\tau}_q^\omega$ with $\hat{\tau}_1\geq  \left(\frac{1}{1-\ln(2+\Delta)}\right)^{\frac{1}{\omega}}$ and $c\geq \frac{\ln(2+\Delta)+1/\hat{\tau}_1^\omega}{1-\ln(2+\Delta)-1/\hat{\tau}_1^\omega}$, we have
\begin{equation*}
    X_{t;\hat{\tau}_q}(s,a) \leq \left(\gamma' + \frac{2}{2+\Delta}\xi\right)G_q.
\end{equation*}
\end{lemma}
\begin{proof}

Observe that $X_{\hat{\tau}_q;\hat{\tau}_q}(s,a) = G_q = \gamma' G_q + (1-\gamma')G_q := \gamma' G_q + \rho_{\hat{\tau}_q} $. We can rewrite the dynamics of $X_{t;\hat{\tau}_q}(s,a)$ as
\begin{equation*}
    X_{t+1;\hat{\tau}_q}(s,a) = (1-\alpha_t)X_{t;\hat{\tau}_q}(s,a) + \alpha_t \gamma'G_q = \gamma'G_q + (1-\alpha_t)\rho_t,
\end{equation*}
where $\rho_{t+1} = (1-\alpha_t)\rho_t$. By the definition of $\rho_t$, we obtain
\begin{align*}
    \rho_t &= (1-\alpha_{t-1})\rho_{t-1} = \dots = (1-\gamma')G_q\prod_{i=\hat{\tau}_q}^{t-1}(1-\alpha_i)\\
    &= (1-\gamma')G_q\prod_{i=\hat{\tau}_q}^{t-1}\left(1-\frac{1}{i^\omega}\right) 
    \overset{\text{(i)}}{\leq} (1-\gamma')G_q\prod_{i=\hat{\tau}_q}^{\hat{\tau}_{q+1}-1}\left(1-\frac{1}{i^\omega}\right)\\
    &\overset{\text{(ii)}}{\leq} (1-\gamma')G_q\exp\left( -\frac{\hat{\tau}_{q+1}-1-\hat{\tau}_q}{(\hat{\tau}_{q+1}-1)^\omega} \right)
    \leq (1-\gamma')G_q\exp\left( -\frac{\hat{\tau}_{q+1}-1-\hat{\tau}_q}{\hat{\tau}_{q+1}^\omega} \right)\\
    &= (1-\gamma')G_q\exp\left( -\frac{\frac{2c}{\kappa}\hat{\tau}_q^\omega-1 }{\hat{\tau}_{q+1}^\omega} \right)
    = (1-\gamma')G_q\exp\left( -\frac{2c}{\kappa}\left(\frac{\hat{\tau}_q}{\hat{\tau}_{q+1}}\right)^\omega + \frac{1}{\hat{\tau}_{q+1}^\omega} \right)\\
    &\overset{\text{(iii)}}{\leq} (1-\gamma')G_q\exp\left( -\frac{2c}{\kappa}\frac{1}{1+\frac{2c}{\kappa}} + \frac{1}{\hat{\tau}_{1}^\omega} \right)
    \overset{\text{(iv)}}{\leq} (1-\gamma')G_q\exp\left( -\frac{c}{1+c} + \frac{1}{\hat{\tau}_{1}^\omega} \right),
\end{align*}
where (i) follows because $\alpha_i$ is decreasing and $t\geq \hat{\tau}_{q+1}$, (ii) follows from Lemma \ref{lem:prodHelp}, (iii) follows because $\hat{\tau}_q\geq\hat{\tau}_1$ and 
\begin{equation*}
    \left(\frac{\hat{\tau}_q}{\hat{\tau}_{q+1}}\right)^\omega \geq \frac{\hat{\tau}_q}{\hat{\tau}_{q+1}} = \frac{\hat{\tau}_q}{\hat{\tau}_{q} + \frac{2c}{\kappa}\hat{\tau}_q^\omega}\geq \frac{1}{1+\frac{2c}{\kappa}},
\end{equation*}
and (iv) follows because $\frac{2c}{\kappa}\geq c$. Next, observing the conditions that $\hat{\tau}_1^\omega\geq\frac{1}{1-\ln(2+\Delta)}$ and $c\geq \frac{1}{1-\ln(2+\Delta)-1/\hat{\tau}_1^\omega}-1$, we have
\begin{equation*}
    \frac{c}{1+c} - \frac{1}{\hat{\tau}_{1}^\omega} \geq \ln(2+\Delta).
\end{equation*}

Thus we have $\rho_t\leq \frac{1-\gamma'}{2+\Delta}G_q$.
Finally, We finish our proof by further observing that $1-\gamma' = 2\xi$.
\end{proof}

Since we have bounded $X_{t;\hat{\tau}_q}(s,a)$ by $\left(\gamma' + \frac{2}{2+\Delta}\xi\right)G_q$ for all $t\geq\hat{\tau}_{q+1}$, it remains to bound $Z_{t;\hat{\tau}_q}(s,a)$ by $\left(1-\frac{2}{2+\Delta}\right)\xi G_q$ for block $q+1$, which will further yield $\norm{u_t^{BA}(s,a)}\leq(\gamma'+\xi)G_q = (1-\xi)G_q = G_{q+1}$ for any $t\in[\hat{\tau}_{q+1},\hat{\tau}_{q+2})$ as desired. Differently from $X_{t;\hat{\tau}_q}(s,a)$ which is a deterministic monotonic sequence, $Z_{t;\hat{\tau}_q}(s,a)$ is stochastic. We need to capture the probability for a bound on $Z_{t;\hat{\tau}_q}(s,a)$ to hold for block $q+1$.
To this end, we introduce a different sequence $\{Z_{t;\hat{\tau}_q}^l (s,a)\}$ given by
\begin{equation}\label{eq:Zl}
    Z^l_{t;\hat{\tau}_q}(s,a) = \sum_{i=\hat{\tau}_q}^{\hat{\tau}_q+l}\alpha_i\prod_{j=i+1}^{t-1}(1-\alpha_j) z_i(s,a) := \sum_{i=\hat{\tau}_q}^{\hat{\tau}_q+l} \phi_i^{q,t-1} z_i(s,a),
\end{equation}
where $\phi_i^{q,t-1} = \alpha_i\prod_{j=i+1}^{t-1}(1-\alpha_j)$. By the definition of $Z_{t;\hat{\tau}_q}(s,a)$, one can check that $Z_{t;\hat{\tau}_q}(s,a) = Z^{t-1-\hat{\tau}_q}_{t;\hat{\tau}_q}(s,a) $. Thus we have
\begin{equation}\label{eq:ZZl}
    Z_{t;\hat{\tau}_q}(s,a) = Z_{t;\hat{\tau}_q}(s,a) - Z_{\hat{\tau}_q;\hat{\tau}_q}(s,a) = \sum_{l=1}^{t-1-\hat{\tau}_q} (Z^l_{t;\hat{\tau}_q}(s,a) - Z^{l-1}_{t;\hat{\tau}_q}(s,a)) + Z^{0}_{t;\hat{\tau}_q}(s,a).
\end{equation}

In the following lemma, we capture an important property of $Z^{l}_{t;\hat{\tau}_q}(s,a)$ defined in \eqref{eq:Zl}.
\begin{lemma}\label{lem:ZlDiff}
For any $t\in[\hat{\tau}_{q+1},\hat{\tau}_{q+2})$ and $1\leq l\leq t-1-\hat{\tau}_q$, $Z^{l}_{t;\hat{\tau}_q}(s,a)$ is a martingale sequence and satisfies 
\begin{equation}\label{eq:ZlvsZ}
    \lvert Z^l_{t;\hat{\tau}_q}(s,a) - Z^{l-1}_{t;\hat{\tau}_q}(s,a) \rvert \leq \frac{2V_{\max}}{\hat{\tau}_q^\omega}.
\end{equation}
\end{lemma}
\begin{proof}
To show the martingale property, we observe that
\begin{align*}
    \mE[Z^l_{t;\hat{\tau}_q}(s,a) - Z^{l-1}_{t;\hat{\tau}_q}(s,a)|\mcF_{\hat{\tau}_q + l -1}] &= \mE[ \phi_{\hat{\tau}_q + l}^{q,t-1} z_{\hat{\tau}_q + l}(s,a)|\mcF_{\hat{\tau}_q + l -1} ]\\
    &= \phi_{\hat{\tau}_q + l}^{q,t-1} \mE[  z_{\hat{\tau}_q + l}(s,a)|\mcF_{\hat{\tau}_q + l -1} ] = 0,
\end{align*}
where the last equation follows from the definition of $z_t(s,a)$.

In addition, based on the definition of $\phi_i^{q,t-1}$ in \eqref{eq:Zl} which requires $i\geq\hat{\tau}_q$, we have
\begin{equation*}
    \phi_i^{q,t-1}=\alpha_i\prod_{j=i+1}^{t-1}(1-\alpha_j)\leq \alpha_i\leq \frac{1}{\hat{\tau}_q^\omega}.
\end{equation*}
Further, since $|F_t|\leq\frac{2R_{\max}}{1-\gamma}=V_{\max}$, we obtain $|z_t(s,a)|=|F_t - \mE[F_t|\mcF_t]|\leq 2V_{\max}$. Thus
\begin{equation*}
    \lvert Z^l_{t;\hat{\tau}_q}(s,a) - Z^{l-1}_{t;\hat{\tau}_q}(s,a) \rvert = \phi_{\hat{\tau}_q + l}^{q,t-1}  |z_{\hat{\tau}_q + l}(s,a)|\leq \frac{2 V_{\max}}{\hat{\tau}_q^\omega}.
\end{equation*}

\end{proof}

Lemma \ref{lem:ZlDiff} guarantees that $Z^{l}_{t;\hat{\tau}_q}(s,a)$ is a martingale sequence, which allows us to apply the following Azuma's inequality.

\begin{lemma}\label{lem:azuma}
\citep{azuma1967weighted} Let $X_0,X_1,\dots,X_n$ be a martingale sequence such that for each $1\leq k\leq n$, 
\begin{equation*}
    |X_k-X_{k-1}| \leq c_k,
\end{equation*}
where the $c_k$ is a constant that may depend on $k$. Then for all $n\geq 1$ and any $\epsilon>0$,
\begin{equation*}
    \mP[|X_n-X_0|>\epsilon] \leq 2\exp\left( -\frac{\epsilon^2}{2\sum_{k=1}^n c_k^2} \right).
\end{equation*}
\end{lemma}

By Azuma's inequality and the relationship between $Z_{t;\hat{\tau}_q}(s,a)$ and $Z^l_{t;\hat{\tau}_q}(s,a)$ in~\eqref{eq:Zl}, we obtain
\begin{align*}
    &\mP\left[ \lvert Z_{t;\hat{\tau}_q}(s,a)  \rvert > \hat{\epsilon}| t\in[\hat{\tau}_{q+1},\hat{\tau}_{q+2})  \right]\\
    &\quad\leq 2\exp\left( -\frac{\hat{\epsilon}^2}{2\sum_{l=1}^{t-\hat{\tau}_q-1} \left( Z^l_{t;\hat{\tau}_q}(s,a) - Z^{l-1}_{t;\hat{\tau}_q}(s,a) \right)^2 + 2(Z^{0}_{t;\hat{\tau}_q}(s,a))^2} \right)\\
    &\quad\overset{\text{(i)}}{\leq} 2\exp\left( -\frac{\hat{\epsilon}^2\hat{\tau}_q^{2\omega}}{8(t-\hat{\tau}_q)V_{\max}^2} \right) \leq 2\exp\left( -\frac{\hat{\epsilon}^2\hat{\tau}_q^{2\omega}}{8(\hat{\tau}_{q+2}-\hat{\tau}_q)V_{\max}^2} \right)\\
    &\quad\overset{\text{(ii)}}{\leq} 2\exp\left( -\frac{\kappa^2\hat{\epsilon}^2\hat{\tau}_q^{\omega}}{32c(c+\kappa)V_{\max}^2} \right) = 2\exp\left( -\frac{\kappa^2\hat{\epsilon}^2\hat{\tau}_q^{\omega}}{32c(c+\kappa)V_{\max}^2} \right),
\end{align*}
where (i) follows from Lemma \ref{lem:ZlDiff}, and (ii) follows because
\begin{equation*}
    \hat{\tau}_{q+2}-\hat{\tau}_q = \frac{2c}{\kappa}\hat{\tau}_{q+1}^\omega+\frac{2c}{\kappa}\hat{\tau}_{q}^\omega = \frac{2c}{\kappa}\left(\hat{\tau}_{q}+\frac{2c}{\kappa}\hat{\tau}_{q}^\omega\right)^\omega+\frac{2c}{\kappa}\hat{\tau}_{q}^\omega\leq \frac{2c}{\kappa}\left(2+\frac{2c}{\kappa}\right)\hat{\tau}_{q}^\omega=\frac{4c(c+\kappa)}{\kappa^2}\hat{\tau}_q^\omega.
\end{equation*}

\subsubsection{Step 4:  Unionizing all blocks and state-action pairs} \label{subsec:proofProp1}
Now we are ready to prove~\Cref{lem:Gq} by taking a union of probabilities over all blocks and state-action pairs. Before that, we introduce the following two preliminary lemmas, which will be used for multiple times in the sequel.

\begin{lemma}\label{lem:unionBound}
Let $\{X_i\}_{i\in\mathcal{I}}$ be a set of random variables. Fix $\epsilon>0$. If for any $i\in\mathcal{I}$, we have $\mP(X_i\leq\epsilon) \geq 1-\delta$, then
\begin{equation*}
    \mP(\forall i\in\mathcal{I}, X_i\leq \epsilon) \geq 1- |\mathcal{I}|\delta.
\end{equation*}
\end{lemma}
\begin{proof}
By union bound, we have
\begin{align*}
    \mP(\forall i\in\mathcal{I}, X_i\leq \epsilon)  = 1-\mP\left(\bigcup_{i\in\mathcal{I}} X_i>\epsilon \right) \geq 1- \sum_{i\in\mathcal{I}}\mP(X_i > \epsilon) \geq 1-|\mathcal{I}|\delta.
\end{align*}

\end{proof}

\begin{lemma}\label{lem:tauHelp}
Fix positive constants $a,b$ satisfying $2ab\ln ab > 1$. If $\tau\geq 2ab\ln ab$, then
\begin{equation*}
    \tau^b\exp\left( -\frac{2\tau}{a} \right) \leq \exp\left( -\frac{\tau}{a} \right).
\end{equation*}
\end{lemma}
\begin{proof}
Let $c=ab$. If $\tau\leq c^2$, we have
\begin{equation*}
    c\ln\tau \leq c\ln c^2 = 2c\ln c\leq \tau.
\end{equation*}
If $\tau\geq c^2$, we have
\begin{equation*}
    c\ln \tau \leq \sqrt{\tau}\ln\tau \leq \sqrt{\tau}\sqrt{\tau} = \tau,
\end{equation*}
where the last inequality follows from $\ln x^2=2\ln x\leq x$. Therefore, we obtain $c\ln\tau = ab\ln\tau \leq \tau$. Thus $\tau^b\leq \exp\left( \frac{\tau}{a} \right)$, which implies this lemma.
\end{proof}

\textbf{Proof of~\Cref{lem:Gq}}\\
Based on the results obtained above, we are ready to prove~\Cref{lem:Gq}. Applying Lemma \ref{lem:unionBound}, we have
\begin{align*}
    &\mP\left[ \forall(s,a), \forall q\in [0,n], \forall t\in[\hat{\tau}_{q+1},\hat{\tau}_{q+2}), \lvert Z_{t;\hat{\tau}_q}(s,a)  \rvert \leq \frac{\Delta}{2+\Delta}\xi G_q \right]\\
    &\quad\geq 1 - \sum_{q=0}^n |\mcs||\mca|(\hat{\tau}_{q+2} - \hat{\tau}_{q+1}) \cdot \mP\left[ \lvert Z_{t;\hat{\tau}_q}(s,a)  \rvert > \frac{\Delta}{2+\Delta}\xi G_q \Big\rvert t\in[\hat{\tau}_{q+1},\hat{\tau}_{q+2})  \right]\\
    &\quad\geq 1 - \sum_{q=0}^n |\mcs||\mca| \frac{2c}{\kappa}\hat{\tau}_{q+1}^\omega \cdot 2\exp\left( -\frac{\kappa^2\left( \frac{\Delta}{2+\Delta}\right)^2\xi^2 G_q^2\hat{\tau}_q^\omega}{32c(c+\kappa)V_{\max}^2} \right)\\
    &\quad\geq 1 - \sum_{q=0}^n |\mcs||\mca| \frac{2c}{\kappa}\left(1+\frac{2c}{\kappa}\right)\hat{\tau}_{q}^\omega \cdot 2\exp\left( -\frac{\kappa^2\left( \frac{\Delta}{2+\Delta}\right)^2\xi^2 G_q^2\hat{\tau}_q^\omega}{32c(c+\kappa)V_{\max}^2} \right)\\
    &\quad\overset{\text{(i)}}{\geq} 1 - \sum_{q=0}^n |\mcs||\mca| \frac{2c}{\kappa}\left(1+\frac{2c}{\kappa}\right)\hat{\tau}_{q}^\omega \cdot 2\exp\left( -\frac{\kappa^2\left( \frac{\Delta}{2+\Delta}\right)^2\xi^2 \sigma^2\epsilon^2\hat{\tau}_q^\omega}{32c(c+\kappa)V_{\max}^2} \right)\\
    &\quad\overset{\text{(ii)}}{\geq} 1 - \frac{4c}{\kappa}\left(1+\frac{2c}{\kappa}\right)\sum_{q=0}^n |\mcs||\mca| \cdot \exp\left( -\frac{\kappa^2\left( \frac{\Delta}{2+\Delta}\right)^2\xi^2 \sigma^2\epsilon^2\hat{\tau}_q^\omega}{64c(c+\kappa)V_{\max}^2} \right)\\
    &\quad\overset{\text{(iii)}}{\geq} 1- \frac{4c(n+1)}{\kappa}\left(1+\frac{2c}{\kappa}\right)|\mcs||\mca| \exp\left( -\frac{\kappa^2\left( \frac{\Delta}{2+\Delta}\right)^2\xi^2 \sigma^2\epsilon^2\hat{\tau}_1^\omega}{64c(c+\kappa)V_{\max}^2} \right),
\end{align*}
where (i) follows because $G_q \geq G_n \geq \sigma\epsilon $, (ii) follows from Lemma \ref{lem:tauHelp} by substituting that $a=\frac{64c(c+\kappa)V_{\max}^2}{\kappa^2\left(\frac{\Delta}{2+\Delta}\right)^2\sigma^2\xi^2\epsilon^2 }, b=1$ and observing
\begin{align*}
    \hat{\tau}_q^\omega&\geq\hat{\tau}_1^\omega\geq \frac{128c(c+\kappa)V_{\max}^2}{\kappa^2\left(\frac{\Delta}{2+\Delta}\right)^2\sigma^2\xi^2\epsilon^2 }\ln\left(\frac{64c(c+\kappa)V_{\max}^2}{\kappa^2\left(\frac{\Delta}{2+\Delta}\right)^2\sigma^2\xi^2\epsilon^2 }\right) = 2ab\ln ab,
\end{align*}
and (iii) follows because $\hat{\tau}_q \geq \hat{\tau}_1$.

Finally, we complete the proof of \Cref{lem:Gq} by observing that $X_{t;\hat{\tau}_q}$ is a deterministic sequence and thus
\begin{align*}
    &\mP\left[ \forall q\in [0,n], \forall t\in[\hat{\tau}_{q+1},\hat{\tau}_{q+2}), \norm{Q^B_t - Q^A_t}\leq G_{q+1} \right]\\
    &\quad\geq \mP\left[ \forall(s,a), \forall q\in [0,n], \forall t\in[\hat{\tau}_{q+1},\hat{\tau}_{q+2}), \lvert Z_{t;\hat{\tau}_q}(s,a)  \rvert \leq \frac{\Delta}{2+\Delta}\xi G_q \right].
\end{align*}



\subsection{Part II: Conditionally bounding $\norm{Q^A_t - Q^*}$} \label{subsec:PartII}
In this part, we upper bound $\norm{Q^A_t - Q^*} $ by a decreasing sequence $\{D_k\}_{k\geq 0}$ block-wisely conditioned on the following two events: fix a positive integer $m$, we define
\begin{align}
    E &:= \left\{ \forall k\in [0,m], \forall t\in[\tau_{k+1},\tau_{k+2}), \norm{Q^B_t - Q^A_t}\leq \sigma D_{k+1} \right\}, \label{eq:eventA}\\
    F &:= \{ \forall k\in [1,m+1], I^A_k\geq c\tau_{k}^\omega \},\label{eq:eventB}
\end{align}
where $I^A_k$ denotes the number of iterations updating $Q^A$ at epoch $k$, $\tau_{k+1}$ is the starting iteration index of the $(k+1)$th block, and $\omega$ is the decay parameter of the polynomial learning rate. Roughly, Event $E$ requires that the difference between the two Q-estimators are bounded appropriately, and Event $F$ requires that $Q^A$ is sufficiently updated in each block. 

\begin{proposition}\label{lem:conditionalBound}
Fix $\epsilon>0, \kappa\in(\ln 2,1)$ and $\Delta\in(0, e^{\kappa}-2)$. Consider synchronous double Q-learning under a polynomial learning rate $\alpha_t = \frac{1}{t^\omega}$ with $\omega\in(0,1)$. Let $\{G_q\}_{q\geq0}, \{\hat{\tau}_q\}_{q\geq0}$ be defined in~\Cref{lem:Gq}. Define $D_k = (1-\beta)^k\frac{V_{\max}}{\sigma}$ with $\beta = \frac{1-\gamma(1+\sigma)}{2}$ and $\sigma = \frac{1-\gamma}{2\gamma}$. Let $\tau_k=\hat{\tau}_k$ for $k\geq0$.
Suppose that $c\geq \frac{\kappa(\ln(2+\Delta) + 1/\tau_1^\omega)}{2(\kappa-\ln(2+\Delta) - 1/\tau_1^\omega)}$ and $\tau_1$ as the finishing time of the first block satisfies 
\begin{equation*}
    \tau_1\geq  \max\left\{\left(\frac{1}{\kappa-\ln(2+\Delta)}\right)^{\frac{1}{\omega}}, \left( \frac{32c(c+\kappa)V_{\max}^2}{\kappa^2\left(\frac{\Delta}{2+\Delta}\right)^2\beta^2\epsilon^2 }\ln \left(\frac{16c(c+\kappa)V_{\max}^2}{\kappa^2\left(\frac{\Delta}{2+\Delta}\right)^2\beta^2\epsilon^2 }\right) \right)^{\frac{1}{\omega}} \right\}.
\end{equation*}
Then for any $m$ such that $D_m\geq\epsilon$, we have
\begin{align*}
    &\mP\left[ \forall k\in [0,m], \forall t\in[\tau_{k+1},\tau_{k+2}), \norm{Q^A_t- Q^*}\leq D_{k+1} |E,F \right]\\
    &\quad\geq 1 - \frac{4c(m+1)}{\kappa}\left(1+\frac{2c}{\kappa}\right)|\mcs||\mca| \exp\left( -\frac{\kappa^2\left( \frac{\Delta}{2+\Delta} \right)^2\beta^2 \epsilon^2\tau_1^{\omega}}{16c(c+\kappa)V_{\max}^2} \right),
\end{align*}
where the events $E,F$ are defined in \eqref{eq:eventA} and \eqref{eq:eventB}, respectively.
\end{proposition}

The proof of~\Cref{lem:conditionalBound} consists of the following four steps.

\subsubsection{Step 1: Designing $\{D_k\}_{k\geq 0}$}
The following lemma establishes the relationship (illustrated in~\Cref{fig:DkGk}) between the block-wise bounds $\{G_q\}_{q\geq 0}$ and $\{D_k\}_{k\geq 0}$ and their block separations, such that Event $E$ occurs with high probability as a result of~\Cref{lem:Gq}. 
\begin{lemma}\label{lem:couple}
Let $\{G_q\}$ be defined in~\Cref{lem:Gq}, and let $D_k = (1-\beta)^k\frac{V_{\max}}{\sigma}$ with $\beta = \frac{1-\gamma(1+\sigma)}{2}$ and $\sigma = \frac{1-\gamma}{2\gamma}$. Then we have
\begin{align*}
    &\mP\left[\forall q\in [0,m], \forall t\in[\hat{\tau}_{q+1},\hat{\tau}_{q+2}), \norm{Q^B_t - Q^A_t}\leq G_{q+1} \right]\\
    &\quad\leq \mP\left[ \forall k\in [0,m], \forall t\in[\tau_{k+1},\tau_{k+2}), \norm{Q^B_t - Q^A_t}\leq \sigma D_{k+1} \right],
\end{align*}
given that $\tau_k = \hat{\tau}_{k }$.
\end{lemma}
\begin{proof}
Based on our choice of $\sigma$, we have
\begin{equation*}
    \beta = \frac{1-\gamma(1+\sigma)}{2} = \frac{1-\gamma\cdot\frac{1+\gamma}{2\gamma}}{2} = \frac{1-\gamma}{4} = \xi.
\end{equation*}
Therefore, the decay rate of $D_k$ is the same as that of $G_q$. Further considering $G_0=\sigma D_0$, we can make the sequence $\{\sigma D_k\}$ as an upper bound of $\{G_q\}$ for any time as long as we set the same starting point and ending point for each epoch.
\end{proof}

In Lemma \ref{lem:couple}, we make $G_k = \sigma D_k$ at any block $k$ and $\xi=\beta=\frac{1-\gamma}{4}$ by careful design of $\sigma$. In fact, one can choose any value of $\sigma\in(0,(1-\gamma)/\gamma)$ and design a corresponding relationship between $\tau_k $ and $ \hat{\tau}_{k }$ as long as the sequence $\{\sigma D_k\}$ can upper bound $\{G_q\}$ for any time. For simplicity of presentation, we keep the design in Lemma \ref{lem:couple}.

\subsubsection{Step 2: Characterizing the dynamics of $Q^A_t(s,a) - Q^*(s,a)$ } 
We characterize the dynamics of the iteration residual $r_{t}(s,a):=Q^A_t(s,a) - Q^*(s,a)$ as an SA algorithm in~\Cref{lem:residualDynamics} below. Since not all iterations contribute to the error propagation due to the random update between the two Q-estimators, we introduce the following notations to label the valid iterations.

\begin{definition}\label{def:TA}
We define $T^A$ as the collection of iterations updating $Q^A$. In addition, we denote $T^A(t_1, t_2)$ as the set of iterations updating $Q^A$ between time $t_1$ and $t_2$. That is, 
\begin{equation*}
    T^A(t_1, t_2) = \left\{ t: t\in [t_1, t_2] \text{ and } t\in T^A \right\}.
\end{equation*}
Correspondingly, the number of iterations updating $Q^A$ between time $t_1$ and $t_2$ is the cardinality of $T^A(t_1, t_2)$ which is denoted as $|T^A(t_1,t_2)|$.
\end{definition}

\begin{lemma}\label{lem:residualDynamics}
Consider double Q-learning in Algorithm \ref{alg:doubleQ}. Then we have
\begin{equation*}
    r_{t+1}(s,a) \!=\! \left\{\begin{aligned}
    & r_t(s,a), \quad t \notin T^A;\\
    & (1\!-\!\alpha_{t}) r_{t}(s,a) \!+\! \alpha_{t} (\mcT Q_{t}^A(s,a)\!-\!Q^*(s,a)) \!+\! \alpha_{t} w_{t}(s,a) \!+\! \alpha_{t}\gamma u_{t}^{BA}(s',a^*), t\in T^A,
    \end{aligned}
    \right. 
\end{equation*}
where $w_{t}(s,a) = \mcT_{t} Q_{t}^A(s,a) - \mcT Q_{t}^A(s,a), u_{t}^{BA}(s,a) =  Q_{t}^B(s,a) -  Q_{t}^A(s,a)$.
\end{lemma}
\begin{proof}
Following from Algorithm \ref{alg:doubleQ} and for $t\in T^A$, we have
\begin{align*}
    &Q_{t+1}^A(s,a)\\
    &\quad= Q_{t}^A(s,a) + \alpha_{t}(R_{t} + \gamma Q_{t}^B(s',a^*) - Q^A_{t}(s,a) )\\
    &\quad= (1-\alpha_{t}) Q_{t}^A(s,a) + \alpha_{t}\left( R_{t} + \gamma Q_{t}^A(s',a^*) \right) + \alpha_{t}\left(\gamma Q_{t}^B(s',a^*) - \gamma Q_{t}^A(s',a^*) \right)\\
    &\quad\overset{\text{(i)}}{=} (1-\alpha_{t}) Q_{t}^A(s,a) + \alpha_{t}\left( \mcT_{t} Q_{t}^A(s,a) + \gamma u_{t}^{BA}(s',a^*) \right)\\
    &\quad= (1-\alpha_{t}) Q_{t}^A(s,a) + \alpha_{t} \mcT Q_{t}^A(s,a) + \alpha_{t} (\mcT_{t} Q_{t}^A(s,a) - \mcT Q_{t}^A(s,a))+ \alpha_{t}\gamma u_{t}^{BA}(s',a^*)\\
    &\quad= (1-\alpha_{t}) Q_{t}^A(s,a) + \alpha_{t} \mcT Q_{t}^A(s,a) + \alpha_{t} w_{t}(s,a) + \alpha_{t}\gamma u_{t}^{BA}(s',a^*),
\end{align*}
where (i) follows because we denote $\mcT_{t} Q_{t}^A(s,a) = R_{t} + \gamma Q_{t}^A(s',a^*)$.
By subtracting $Q^*$ from both sides, we complete the proof.
\end{proof}

\subsubsection{Step 3: Constructing sandwich bounds on $r_t(s,a)$} 
We provide upper and lower bounds on $r_t$ by constructing two sequences $Y_{t;\tau_k}$ and $W_{t;\tau_k}$ in the following lemma.
\begin{lemma}\label{lem:rtSandwich}
Let $\tau_k$ be such that $\norm{r_t}\leq D_k$ for all $t\geq\tau_k$. Suppose that we have $\norm{u_t^{BA}}\leq \sigma D_k$ with $\sigma = \frac{1-\gamma}{2\gamma}$ for all $t\geq\tau_k$. Define $W_{t;\tau_k}(s,a)$ as
\begin{equation*}
    W_{t+1;\tau_k}(s,a) = \left\{\begin{aligned}
    &W_{t;\tau_k}(s,a), \quad t\notin T^A;\\
    &(1-\alpha_t)W_{t;\tau_k}(s,a) + \alpha_t w_t(s,a), \quad t\in T^A,
    \end{aligned}\right.
\end{equation*}
where $W_{\tau_k;\tau_k}(s,a) = 0$ and define $Y_{t;\tau_k}(s,a)$ as
\begin{equation*}
    Y_{t+1;\tau_k}(s,a) = \left\{\begin{aligned}
    &Y_{t;\tau_k}(s,a), \quad t\notin T^A;\\
    &(1-\alpha_t)Y_{t;\tau_k}(s,a) + \alpha_t \gamma''D_k, \quad t\in T^A,
    \end{aligned}\right.
\end{equation*}
where $Y_{\tau_k;\tau_k}(s,a) = D_k$ and $\gamma''=\gamma(1+\sigma)$.
Then for any $t\geq\tau_k$ and state-action pair $(s,a)$, we have
\begin{equation*}
    -Y_{t;\tau_k}(s,a) + W_{t;\tau_k}(s,a) \leq r_t(s,a) \leq Y_{t;\tau_k}(s,a) + W_{t;\tau_k}(s,a).
\end{equation*}
\end{lemma}
\begin{proof}
We proceed the proof by induction. For the initial condition $t=\tau_k$, we have $\norm{r_t(s,a)}\leq D_k$, and thus it holds that $-D_k \leq r_{\tau_k}(s,a) \leq D_k$. We assume the sandwich bound holds for time $t\geq\tau_k$. It remains to check whether this bound holds for $t+1$.

If $t\notin T^A$, then $r_{t+1}(s,a) = r_t(s,a), W_{t+1;\tau_k}(s,a)=W_{t;\tau_k}(s,a), Y_{t+1;\tau_k}(s,a)=Y_{t;\tau_k}(s,a)$. Thus the sandwich bound still holds.

If $t\in T^A$, we have
\begin{align*}
        r_{t+1}(s,a) &= (1-\alpha_t) r_t(s,a) + \alpha_t( \mcT Q_t^A(s,a) - Q^*(s,a) ) + \alpha_t w_t(s,a) + \alpha_t\gamma u_t^{BA}(s',a^*)\\
        &\leq  (1-\alpha_t) (Y_{t;\tau_k}(s,a) + W_{t;\tau_k}(s,a)) + \alpha_t\norm{ \mcT Q_t^A - Q^*}\\
        &\quad + \alpha_t w_t(s,a) + \alpha_t\gamma \norm{u_t^{BA}}\\
        &\overset{\text{(i)}}{\leq} (1-\alpha_t) (Y_{t;\tau_k}(s,a) + W_{t;\tau_k}(s,a)) + \alpha_t \gamma \norm{r_t}\\
        &\quad + \alpha_t w_t(s,a) + \alpha_t\gamma \norm{u_t^{BA}}\\
        &\overset{\text{(ii)}}{\leq} (1-\alpha_t) Y_{t;\tau_k}(s,a) + \alpha_t \gamma(1+\sigma)D_k + (1-\alpha_t) W_{t;\tau_k}(s,a) + \alpha_t w_t(s,a)\\
        &\leq Y_{t+1;\tau_k}(s,a) + W_{t+1;\tau_k}(s,a),
\end{align*}
where (i) follows from the contraction property of the Bellman operator, and (ii) follows from the condition $\norm{u_t^{BA}}\leq \sigma D_k$.

Similarly, we can bound the other direction as
\begin{align*}
        r_{t+1}(s,a) &= (1-\alpha_t) r_t(s,a) + \alpha_t( \mcT Q_t^A(s,a) - Q^*(s,a) ) + \alpha_t w_t(s,a) + \alpha_t\gamma u_t^{BA}(s',a^*)\\
        &\geq  (1-\alpha_t) (-Y_{t;\tau_k}(s,a) + W_{t;\tau_k}(s,a)) - \alpha_t\norm{ \mcT Q_t^A - Q^*}\\
        &\quad + \alpha_t w_t(s,a) - \alpha_t\gamma \norm{u_t^{BA}}\\
        &\geq (1-\alpha_t) (Y_{t;\tau_k}(s,a) + W_{t;\tau_k}(s,a)) - \alpha_t \gamma \norm{r_t}\\
        &\quad + \alpha_t w_t(s,a) - \alpha_t\gamma \norm{u_t^{BA}}\\
        &\geq -(1-\alpha_t) Y_{t;\tau_k}(s,a) - \alpha_t \gamma(1+\sigma)D_k + (1-\alpha_t) W_{t;\tau_k}(s,a) + \alpha_t w_t(s,a)\\
        &\geq -Y_{t+1;\tau_k}(s,a) + W_{t+1;\tau_k}(s,a).
\end{align*}

\end{proof}

\subsubsection{Step 4: Bounding $Y_{t;\tau_k}(s,a)$ and $W_{t;\tau_k}(s,a)$ for epoch $k+1$} \label{subsec:proofProp2}
Similarly to Steps 3 and 4 in Part I, we conditionally bound $\norm{r_t}\leq D_k$ for $t\in [\tau_{k}, \tau_{k+1})$ and $k=0,1,2,\dots$ by the induction arguments followed by the union bound.
We first bound $Y_{t;\tau_k}(s,a)$ and $W_{t;\tau_k}(s,a)$ in \Cref{lem:Yt} and~\Cref{lem:WlDiff}, respectively. 

\begin{lemma}\label{lem:Yt}
Fix $\kappa\in(\ln 2,1)$ and $\Delta\in(0, e^{\kappa}-2)$. Let $\{D_k\}$ be defined in Lemma \ref{lem:couple}. Consider synchronous double Q-learning using a polynomial learning rate $\alpha_t = \frac{1}{t^\omega}$ with $\omega\in(0,1)$. Suppose that $Y_{t;\tau_k}(s,a) \leq D_k$ for any $t \geq \tau_k$. At block $k$, we assume that there are at least $c\tau_k^\omega$ iterations updating $Q^A$, i.e., $|T^A(\tau_k,\tau_{k+1})|\geq c\tau_k^\omega$. Then for any $t\in[\tau_{k+1},\tau_{k+2})$, we have
\begin{equation*}
    Y_{t;\tau_k}(s,a) \leq \left(\gamma'' + \frac{2}{2+\Delta}\beta\right)D_k.
\end{equation*}
\end{lemma}
\begin{proof}
Since we have defined $\tau_k=\hat{\tau}_k$ in Lemma \ref{lem:couple}, we have $\tau_{k+1} = \tau_k + \frac{2c}{\kappa}\tau_k^\omega$. 

Observe that $Y_{\tau_k;\tau_k}(s,a) = D_k = \gamma'' D_k + (1-\gamma'')D_k := \gamma'' D_k + \rho_{\tau_k} $. We can rewrite the dynamics of $Y_{t;\tau_k}(s,a)$ as
\begin{equation*}
    Y_{t+1;\tau_k}(s,a) = \left\{\begin{aligned}
    & Y_{t;\tau_k}(s,a), \quad t\notin T^A\\
    &(1-\alpha_t)Y_{t;\tau_k}(s,a) + \alpha_t \gamma''D_k = \gamma''D_k + (1-\alpha_t)\rho_t, \quad t\in T^A
    \end{aligned}\right.
\end{equation*}
where $\rho_{t+1} = (1-\alpha_t)\rho_t$ for $t\in T^A$. By the definition of $\rho_t$, we obtain
\begin{align}
    \rho_t &= \rho_{\tau_k}\prod_{i\in T^A(\tau_k, t-1)}(1-\alpha_i) = (1-\gamma'')D_k\prod_{i\in T^A(\tau_k, t-1)}(1-\alpha_i)\nonumber\\
    &= (1-\gamma'')D_k\prod_{i\in T^A(\tau_k, t-1)}\left(1-\frac{1}{i^\omega}\right) 
    \overset{\text{(i)}}{\leq} (1-\gamma'')D_k\prod_{i\in T^A(\tau_k, \tau_{k+1}-1)}\left(1-\frac{1}{i^\omega}\right) \label{eq:issue1}\\
    &\overset{\text{(ii)}}{\leq} (1-\gamma'')D_k\prod_{i=\tau_{k+1}-c\tau_k^\omega}^{\tau_{k+1}-1}\left(1-\frac{1}{i^\omega}\right)
    \overset{\text{(iii)}}{\leq} (1-\gamma'')D_k\exp\left( -\frac{c\tau_k^\omega-1}{(\tau_{k+1}-1)^\omega} \right) \nonumber\\
    &\leq (1-\gamma'')D_k\exp\left( -\frac{c\tau_k^\omega-1}{\tau_{k+1}^\omega} \right) 
    = (1-\gamma'')D_k\exp\left( -c\left(\frac{\tau_k}{\tau_{k+1}}\right)^\omega + \frac{1}{\tau_{k+1}^\omega} \right) \nonumber\\
    &\overset{\text{(iv)}}{\leq} (1-\gamma'')D_k\exp\left( -\frac{c}{1+\frac{2c}{\kappa}} + \frac{1}{\tau_{1}^\omega} \right), \nonumber
\end{align}
where (i) follows because $\alpha_i<1$ and $t\geq \tau_{k+1}$, (ii) follows because $|T^A(\tau_{k}, \tau_{k+1}-1)|\geq c\tau_k^\omega$ where $T^A(t_1,t_2)$ and $|T^A(t_1,t_2)|$ are defined in Definition \ref{def:TA}, (iii) follows from Lemma \ref{lem:tauHelp}, and (iv) holds because $\tau+k\geq\tau_1$ and 
\begin{equation*}
    \left(\frac{\tau_k}{\tau_{k+1}}\right)^\omega \geq \frac{\tau_k}{\tau_{k+1}} = \frac{\tau_k}{\tau_{k} + \frac{2c}{\kappa}\tau_k^\omega}\geq \frac{1}{1+\frac{2c}{\kappa}}.
\end{equation*}
Next we check the value of the power $-\frac{c}{1+\frac{2c}{\kappa}} + \frac{1}{\tau_{1}^\omega}$. Since $\kappa\in(\ln 2,1)$ and $\Delta\in(0, e^{\kappa}-2)$, we have $\ln (2+\Delta)\in (0, \kappa)$. Further, observing $\tau_1^\omega > \frac{1}{\kappa-\ln(2+\Delta)}$, we obtain $\ln(2+\Delta) + \frac{1}{\tau_1^\omega} \in (0,\kappa)$. Last, since $c\geq\frac{\kappa}{2}\left( \frac{1}{1-\frac{\ln(2+\Delta) + 1/\tau_1^\omega}{\kappa}} - 1 \right)=\frac{\kappa(\ln(2+\Delta) + 1/\tau_1^\omega)}{2(\kappa-\ln(2+\Delta) - 1/\tau_1^\omega)}$, we have $-\frac{c}{1+\frac{2c}{\kappa}} + \frac{1}{\tau_{1}^\omega}\leq -\ln(2+\Delta)$.

Thus, we have $\rho_t\leq \frac{1-\gamma''}{2+\Delta}D_k$.
Finally, we finish our proof by further observing that $1-\gamma'' = 2\beta$.

\end{proof}

It remains to bound $|W_{t;\tau_k}(s,a)|\leq \left(1-\frac{2}{2+\Delta}\right)\beta D_k$ for $t\in [\tau_{k+1},\tau_{k+2})$. Combining the bounds of $Y_{t;\tau_k}$ and $W_{t;\tau_k}$ yields $(\gamma''+\beta)D_k = (1-\beta)D_k=D_{k+1}$. Since $W_{t;\tau_k}$ is stochastic, we need to derive the probability for the bound to hold. To this end, we first rewrite the dynamics of $W_{t;\tau_k}$ defined in Lemma \ref{lem:rtSandwich} as
\begin{equation*}
    W_{t;\tau_k}(s,a) = \sum_{i\in T^A(\tau_k, t-1)} \alpha_i\underset{j\in T^A(i+1, t-1)}{\Pi} (1-\alpha_j)w_i(s,a).
\end{equation*}
Next, we introduce a new sequence $\{W_{t;\tau_k}^l(s,a)\}$ as
\begin{equation*}
    W^l_{t;\tau_k}(s,a) = \sum_{i\in T^A(\tau_k, \tau_k+l)} \alpha_i\underset{j\in T^A(i+1, t-1)}{\Pi} (1-\alpha_j)w_i(s,a).
\end{equation*}
Thus we have $W_{t;\tau_k}(s,a) = W^{t-1-\tau_k}_{t;\tau_k}(s,a)$. Then we have the following lemma.

\begin{lemma}\label{lem:WlDiff}
For any $t\in[\tau_{k+1}, \tau_{k+2}]$ and $1\leq l \leq t-\tau_k-1$, $\{W_{t;\tau_k}^l(s,a)\}$ is a martingale sequence and satisfies 
\begin{equation*}
    \lvert W^l_{t;\tau_k}(s,a) - W^{l-1}_{t;\tau_k}(s,a) \rvert \leq  \frac{V_{\max}}{\tau_k^\omega}.
\end{equation*}
\end{lemma}
\begin{proof}
Observe that
\begin{equation*}
    W^l_{t;\tau_k}(s,a) - W^{l-1}_{t;\tau_k}(s,a) = \left\{\begin{aligned}
    &0, \quad \tau_k+l-1\notin T^A;\\
    &\alpha_{\tau_k+l}\underset{j\in T^A(\tau_k+l+1, t-1)}{\Pi} (1-\alpha_j)w_{\tau_k+l}(s,a), \quad \tau_k+l-1\in T^A.
    \end{aligned}
    \right.
\end{equation*}
Since $\mE[w_t|\mcF_{t-1}]=0$, we have
\begin{align*}
    \mE\left[ W^l_{t;\tau_k}(s,a) - W^{l-1}_{t;\tau_k}(s,a) | \mcF_{\tau_k+l-1} \right]=0.
\end{align*}
Thus $\{W_{t;\tau_k}^l(s,a)\}$ is a martingale sequence.
In addition, since $l\geq 1$ and $\alpha_t\in (0,1)$, we have
\begin{equation*}
    \alpha_{\tau_k+l}\underset{j\in T^A(\tau_k+l+1, t-1)}{\Pi} (1-\alpha_j)\leq \alpha_{\tau_k+l}\leq \alpha_{\tau_k} = \frac{1}{\tau_k^\omega}.
\end{equation*}
Further, we obtain $|w_t(s,a)| = |\mcT_{t} Q_{t}^A(s,a) - \mcT Q_{t}^A(s,a)| \leq\frac{2Q_{\max}}{1-\gamma} = V_{\max}$. Thus
\begin{equation*}
    \lvert W^l_{t;\tau_k}(s,a) - W^{l-1}_{t;\tau_k}(s,a) \rvert \leq \alpha_{\tau_k+l}|w_{\tau_k+l}(s,a)| \leq  \frac{V_{\max}}{\tau_k^\omega}.
\end{equation*}

\end{proof}

Next, we bound $W_{t;\tau_k}(s,a)$. Fix $\tilde{\epsilon}>0$. Then for any $t\in[\tau_{k+1},\tau_{k+2})$, we have
\begin{align*}
    &\mP\left[ |W_{t;\tau_k}(s,a)|>\tilde{\epsilon} | t\in[\tau_{k+1},\tau_{k+2}),E,F \right]\\
    &\quad\overset{\text{(i)}}{\leq} 2\exp\left( \frac{-\tilde{\epsilon}^2}{2\underset{l:\tau_k+l-1\in T^A(\tau_k, t-1)}{\sum}\left( W^l_{t;\tau_k}(s,a) - W^{l-1}_{t;\tau_k}(s,a) \right)^2 + 2(W^{\min(T^A(\tau_k, t-1))}_{t;\tau_k}(s,a))^2  } \right)\\
    &\quad\overset{\text{(ii)}}{\leq} 2\exp\left( -\frac{\hat{\epsilon}^2\tau_k^{2\omega}}{2(|T^A(\tau_k,t-1)|+1)V_{\max}^2} \right) \overset{\text{(iii)}}{\leq} 2\exp\left( -\frac{\tilde{\epsilon}^2\tau_k^{2\omega}}{2(t+1-\tau_k)V_{\max}^2} \right)\\
    &\quad\leq 2\exp\left( -\frac{\tilde{\epsilon}^2\tau_k^{2\omega}}{2(\tau_{k+2}-\tau_k)V_{\max}^2} \right) \overset{\text{(iv)}}{\leq} 2\exp\left( -\frac{\kappa^2\tilde{\epsilon}^2\tau_k^{\omega}}{8c(c+\kappa)V_{\max}^2} \right),
\end{align*}
where (i) follows from Lemma \ref{lem:azuma}, (ii) follows from Lemma \ref{lem:WlDiff}, (iii) follows because $|T^A(t_1,t_2)|\leq t_2 - t_1 + 1$ and (iv) holds because
\begin{equation*}
    \tau_{k+2} - \tau_k = \frac{2c}{\kappa}\tau_{k+1}^\omega + \frac{2c}{\kappa}\tau_k^\omega = \frac{2c}{\kappa}\left( \tau_k + \frac{2c}{\kappa}\tau_k^\omega \right)^\omega + \frac{2c}{\kappa}\tau_k^\omega \leq \frac{4c(c+\kappa)}{\kappa^2}\tau_k^\omega.
\end{equation*}

\textbf{Proof of~\Cref{lem:conditionalBound}}\\
Now we bound $\norm{r_t}$ by combining the bounds of $Y_{t;\tau_k}$ and $W_{t;\tau_k}$.
Applying the union bound in Lemma \ref{lem:unionBound} yields
\begin{align}
    &\mP\left[ \forall(s,a), \forall k\in [0,m], \forall t\in[\tau_{k+1},\tau_{k+2}), \lvert W_{t;\tau_k}(s,a)  \rvert \leq \frac{\Delta}{2+\Delta}\beta D_k|E,F \right] \nonumber\\
    &\quad\geq 1 - \sum_{k=0}^m |\mcs||\mca|(\tau_{k+2}-\tau_{k+1}) \cdot \mP\left[ \lvert W_{t;\tau_k}(s,a)  \rvert > \frac{\Delta}{2+\Delta}\beta D_k \Big\rvert t\in[\tau_{k+1},\tau_{k+2}),E,F  \right] \nonumber\\
    &\quad\geq 1 - \sum_{k=0}^m |\mcs||\mca| \frac{2c}{\kappa}\tau_{k+1}^\omega \cdot 2\exp\left( -\frac{\kappa^2\left( \frac{\Delta}{2+\Delta} \right)^2\beta^2 D_k^2\tau_k^{\omega}}{8c(c+\kappa)V_{\max}^2} \right) \nonumber\\
    &\quad\geq 1 - \sum_{k=0}^m |\mcs||\mca| \frac{2c}{\kappa}\left(1+\frac{2c}{\kappa}\right)\tau_{k}^\omega \cdot 2\exp\left( -\frac{\kappa^2\left( \frac{\Delta}{2+\Delta} \right)^2\beta^2 D_k^2\tau_k^{\omega}}{8c(c+\kappa)V_{\max}^2} \right) \nonumber\\
    &\quad\overset{\text{(i)}}{\geq} 1 - \sum_{k=0}^m |\mcs||\mca| \frac{2c}{\kappa}\left(1+\frac{2c}{\kappa}\right)\tau_{k}^\omega \cdot 2\exp\left( -\frac{\kappa^2\left( \frac{\Delta}{2+\Delta} \right)^2\beta^2 \epsilon^2\tau_k^{\omega}}{8c(c+\kappa)V_{\max}^2} \right)\label{eq:issue2}\\
    &\quad\overset{\text{(ii)}}{\geq} 1 - \frac{4c}{\kappa}\left(1+\frac{2c}{\kappa}\right)\sum_{k=0}^m |\mcs||\mca| \cdot \exp\left( -\frac{\kappa^2\left( \frac{\Delta}{2+\Delta} \right)^2\beta^2 \epsilon^2\tau_k^{\omega}}{16c(c+\kappa)V_{\max}^2} \right) \nonumber\\
    &\quad\geq 1 - \frac{4c(m+1)}{\kappa}\left(1+\frac{2c}{\kappa}\right)|\mcs||\mca| \exp\left( -\frac{\kappa^2\left( \frac{\Delta}{2+\Delta} \right)^2\beta^2 \epsilon^2\tau_1^{\omega}}{16c(c+\kappa)V_{\max}^2} \right), \nonumber
\end{align}
where (i) follows because $D_k\geq D_m\geq \epsilon$, and (ii) follows from Lemma \ref{lem:tauHelp} by substituting $a=\frac{16c(c+\kappa)V_{\max}^2}{\kappa^2\left(\frac{\Delta}{2+\Delta}\right)^2\beta^2\epsilon^2 }, b=1$ and observing that
\begin{align*}
    \tau_k^{\omega}&\geq\hat{\tau}_1^{\omega}\geq \frac{32c(c+\kappa)V_{\max}^2}{\kappa^2\left(\frac{\Delta}{2+\Delta}\right)^2\beta^2\epsilon^2 }\ln \left(\frac{16c(c+\kappa)V_{\max}^2}{\kappa^2\left(\frac{\Delta}{2+\Delta}\right)^2\beta^2\epsilon^2 }\right) = 2ab\ln ab.
\end{align*}

Note that $Y_{t;\tau_k}(s,a)$ is deterministic. We complete this proof by observing that
\begin{align*}
    &\mP\left[ \forall k\in [0,m], \forall t\in[\tau_{k+1},\tau_{k+2}), \norm{Q^A_t - Q^*}\leq D_{k+1}  | E,F\right]\\
    &\quad\geq \mP\left[ \forall(s,a), \forall k\in [0,m], \forall t\in[\tau_{k+1},\tau_{k+2}), \lvert W_{t;\tau_k}(s,a)  \rvert \leq  \frac{\Delta}{2+\Delta}\beta D_k|E,F \right].
\end{align*}

\subsection{Part III: Bounding $\norm{Q^A_t - Q^* }$} \label{subsec:proofThm1}
We combine the results in the first two parts, and provide a high probability bound on $\norm{r_t}$ with further probabilistic arguments, which exploit the high probability bounds on $\mP(E)$ in~\Cref{lem:Gq} and $\mP(F)$ in the following lemma. 
\begin{lemma}\label{lem:halfQA}
Let the sequence $\tau_k$ be the same as given in Lemma \ref{lem:couple}, i.e. $\tau_{k+1} = \tau_k + \frac{2c}{\kappa}\tau_k^\omega$ for $k\geq 1$. Then we have
\begin{equation*}
    \mP\left[\forall k\in [1,m], I^A_k\geq c\tau_{k}^\omega \right] \geq 1- m \exp\left( -\frac{(1-\kappa)^2c\tau_1^\omega}{\kappa} \right).
\end{equation*}
where $I^A_k$ denotes the number of iterations updating $Q^A$ at epoch $k$. \end{lemma}
\begin{proof}
The event updating $Q^A$ is a binomial random variable. To be specific, at iteration $t$ we define
\begin{equation*}
    J^A_t = \left\{ \begin{aligned} & 1, \quad\text{updating } Q^A;\\
    & 0, \quad\text{updating } Q^B.
    \end{aligned}
    \right.
\end{equation*}
Clearly, the events are independent across iterations. Therefore, for a given epoch $[\tau_k, \tau_{k+1})$, $I^A_k = \sum_{t=\tau_k}^{\tau_{k+1}-1} J^A_t$ is a binomial random variable satisfying the distribution $Binomial(\tau_{k+1}-\tau_k, 0.5)$. 
In the following, we use the tail bound of a binomial random variable. That is, if a random variable $X\sim Binomial(n,p)$, by Hoeffding's inequality we have $\mP(X\leq x)\leq \exp\left(-\frac{2(np-x)^2}{n}\right)$ for $x< np$, which implies $\mP(X\leq \kappa np)\leq \exp\left(-2np^2(1-\kappa)^2\right)$ for any fixed $\kappa\in(0,1)$.

If $k=0$, $I^A_0\sim Binomial(\tau_1, 0.5)$. Thus the tail bound yields
\begin{equation*}
    \mP \left[I^A_0\leq \frac{\kappa}{2}\cdot\tau_1\right] \leq \exp\left( -\frac{(1-\kappa)^2\tau_1}{2} \right).
\end{equation*}
If $k\geq1$, since $\tau_{k+1}-\tau_k = \frac{2c}{\kappa}\tau_k^\omega $, we have $I^A_k\sim Binomial\left( \frac{2c}{\kappa}\tau_k^\omega, 0.5 \right)$. Thus the tail bound of a binomial random variable gives
\begin{equation*}
    \mP \left[I^A_k\leq \frac{\kappa}{2}\cdot \frac{2c}{\kappa}\tau_k^\omega\right] \leq \exp\left( -\frac{(1-\kappa)^2c\tau_k^\omega}{\kappa} \right).
\end{equation*}
Then by the union bound, we have
\begin{align*}
    \mP \left[\forall k\in[1,m], I^A_k\geq c\tau_k^\omega\right]
    &=\mP \left[\forall k\in[1,m], I^A_k\geq \frac{\kappa}{2}\cdot \frac{2c}{\kappa}\tau_k^\omega\right]\\
    &\geq 1 - \sum_{k=1}^m\exp\left( -\frac{(1-\kappa)^2c\tau_k^\omega}{\kappa} \right)\\
    &\geq 1- m \exp\left( -\frac{(1-\kappa)^2c\tau_1^\omega}{\kappa} \right).
\end{align*}
\end{proof}
We further give the following~\Cref{lem:totalIter} and~\Cref{lem:iteration} before proving~\Cref{thm:syncDQ}.
\Cref{lem:totalIter} characterizes the number of blocks to achieve $\epsilon$-accuracy given $D_k$ defined in Lemma \ref{lem:couple}.

\begin{lemma}\label{lem:totalIter}
Let $D_{k+1}=(1-\beta)D_k$ with $\beta = \frac{1-\gamma}{4}, D_0 = \frac{2\gamma V_{\max}}{1-\gamma}$. Then for $m\geq\frac{4}{1-\gamma}\ln \frac{2\gamma V_{\max}}{\epsilon(1-\gamma)}$, we have $D_m \leq \epsilon$.
\end{lemma}
\begin{proof}
By the definition of $D_k$, we have $D_k = \left( 1 - \beta \right)^k D_0$.
Then we obtain
\begin{equation*}
    D_k\leq\epsilon \Longleftrightarrow \left( 1 - \beta \right)^k D_0 \leq \epsilon \Longleftrightarrow \frac{1}{(1-\beta)^k} \geq \frac{D_0}{\epsilon} \Longleftrightarrow k \geq \frac{\ln(D_0/\epsilon)}{\ln(1/(1-\beta))}.
\end{equation*}
Further observe that $\ln \frac{1}{1-x}\leq x$ if $x\in(0,1)$. Thus we have
\begin{equation*}
    k  \geq\frac{1}{\beta}\ln \frac{D_0}{\epsilon} = \frac{4}{1-\gamma}\ln \frac{2\gamma V_{\max}}{\epsilon(1-\gamma)}.
\end{equation*}

\end{proof}

From the above lemma, it suffices to find the starting time at epoch $m^*=\left\lceil \frac{4}{1-\gamma}\ln \frac{2\gamma V_{\max}}{\epsilon(1-\gamma)}\right\rceil$.

The next lemma is useful to calculate the total iterations given the initial epoch length and number of epochs.

\begin{lemma}\label{lem:iteration}
\citep[Lemma 32]{even2003learning} Consider a sequence $\{x_k\}$ satisfying 
\begin{equation*}
    x_{k+1} = x_k + c x_k^\omega = x_1 + \sum_{i=1}^k c x_i^\omega.
\end{equation*}
Then for any constant $\omega\in(0,1)$, we have 
\begin{equation*}
    x_k = O\left( (x_1^{1-\omega} + c k)^\frac{1}{1-\omega} \right) = O\left( x_1 + (ck)^{\frac{1}{1-\omega}} \right).
\end{equation*}
\end{lemma}

\textbf{Proof of Theorem \ref{thm:syncDQ}}\\
Now we are ready to prove Theorem \ref{thm:syncDQ} based on the results obtained so far.\\
Let $m^*=\Big\lceil \frac{4}{1-\gamma}\ln \frac{2\gamma V_{\max}}{\epsilon(1-\gamma)}\Big\rceil$, then $G_{m^*-1}\geq\sigma\epsilon, D_{m^*-1}\geq \epsilon$. Thus we obtain
\begin{align*}
    &\mP(\norm{Q^A_{\tau_{m^*}}(s,a) - Q^* } \leq \epsilon)\\
    & \geq\mP\left[   \forall k\in [0,m^*- 1], \forall t\in[\tau_{k+1},\tau_{k+2}), \norm{Q^A_t  - Q^* }\leq D_{k+1} \right]\\
    & = \mP\left[   \forall k\in [0,m^*- 1], \forall t\in[\tau_{k+1},\tau_{k+2}), \norm{Q^A_t  - Q^* }\leq D_{k+1} |E,F \right]\cdot\mP(E\cap F)\\
    & \geq \mP\left[   \forall k\in [0,m^*-1], \forall t\in[\tau_{k+1},\tau_{k+2}), \norm{Q^A_t  - Q^* }\leq D_{k+1} |E,F \right]\\
    &\quad \cdot (\mP(E)+\mP(F)-1)\\
    &\overset{\text{(i)}}{\geq}\mP\left[   \forall k\in [0,m^*- 1], \forall t\in[\tau_{k+1},\tau_{k+2}), \norm{Q^A_t  - Q^* }\leq D_{k+1} |E,F \right]\\
    &\quad\cdot \left(\mP\left[   \forall q\in [0, m^*- 1], \forall t\in[\hat{\tau}_{q+1},\hat{\tau}_{q+2}), \norm{Q^B_t  - Q^A_t }\leq G_{q+1} \right]  + \mP(F) - 1\right)\\
    &\overset{\text{(ii)}}{\geq}\left[ 1 - \frac{4cm^*}{\kappa}\left(1+\frac{2c}{\kappa}\right)|\mcs||\mca| \exp\left( -\frac{\kappa^2\left( \frac{\Delta}{2+\Delta} \right)^2\beta^2 \epsilon^2\tau_1^{\omega}}{16c(c+\kappa)V_{\max}^2} \right) \right]\\
    &\quad\cdot\!\left[ 1 \!-\! \frac{4cm^*}{\kappa}\left(1\!+\!\frac{2c}{\kappa}\right)|\mcs||\mca| \exp\left( -\frac{\kappa^2\left( \frac{\Delta}{2+\Delta}\right)^2\xi^2 \sigma^2\epsilon^2\hat{\tau}_1^\omega}{64c(c+\kappa)V_{\max}^2} \right) \!-\! m^*\! \exp\!\left( -\frac{(1\!-\!\kappa)^2c\hat{\tau}_1^{\omega}}{\kappa} \right) \right]\\
    &\geq 1- \frac{4cm^*}{\kappa}\left(1+\frac{2c}{\kappa}\right)|\mcs||\mca| \exp\left( -\frac{\kappa^2\left( \frac{\Delta}{2+\Delta} \right)^2\beta^2 \epsilon^2\tau_1^{\omega}}{16c(c+\kappa)V_{\max}^2} \right)\\
    &\quad - \frac{4cm^*}{\kappa}\left(1\!+\!\frac{2c}{\kappa}\right)|\mcs||\mca| \exp\left( -\frac{\kappa^2\left( \frac{\Delta}{2+\Delta}\right)^2\xi^2 \sigma^2\epsilon^2\hat{\tau}_1^\omega}{64c(c+\kappa)V_{\max}^2} \right) - m^* \exp\left( -\frac{(1-\kappa)^2c\hat{\tau}_1^{\omega}}{\kappa} \right)\\
    &\overset{\text{(iii)}}{\geq} 1- \frac{12cm^*}{\kappa}\left(1+\frac{2c}{\kappa}\right)|\mcs||\mca| \exp\left( -\frac{\kappa^2(1-\kappa)^2\left( \frac{\Delta}{2+\Delta} \right)^2\xi^2 \sigma^2\epsilon^2\hat{\tau}_1^{\omega}}{64c(c+\kappa)V_{\max}^2} \right),
\end{align*}
where (i) follows from~\Cref{lem:couple}, (ii) follows from~\Cref{lem:Gq} and~\ref{lem:conditionalBound} and (iii) holds due to the fact that
\begin{align*}
    \frac{4cm^*}{\kappa}\left(1+\frac{2c}{\kappa}\right)|\mcs||\mca| \!=\! \max&\left\{ \frac{4cm^*}{\kappa}\left(1+\frac{2c}{\kappa}\right)|\mcs||\mca|, m^* \right\},\\
    \frac{\kappa^2(1\!-\!\kappa)^2\left( \frac{\Delta}{2+\Delta} \right)^2\xi^2 \sigma^2\epsilon^2\hat{\tau}_1^{\omega}}{64c(c+\kappa)V_{\max}^2}\!\leq\! \min&\left\{ \frac{\kappa^2\!\left( \frac{\Delta}{2+\Delta} \right)^2\!\beta^2 \epsilon^2\hat{\tau}_1^{\omega}}{16c(c+\kappa)V_{\max}^2}, \frac{(1\!-\!\kappa)^2\hat{\tau}_1^{\omega}}{\kappa}, \frac{\kappa^2\!\left( \frac{\Delta}{2+\Delta} \right)^2\!\xi^2 \sigma^2\epsilon^2\hat{\tau}_1^{\omega}}{64c(c+\kappa)V_{\max}^2}\right\}.
\end{align*}

By setting 
\begin{equation*}
    1- \frac{12cm^*}{\kappa}\left(1+\frac{2c}{\kappa}\right)|\mcs||\mca| \exp\left( -\frac{\kappa^2(1-\kappa)^2\left( \frac{\Delta}{2+\Delta} \right)^2\xi^2 \sigma^2\epsilon^2\hat{\tau}_1^{\omega}}{64c(c+\kappa)V_{\max}^2} \right) \geq 1-\delta,
\end{equation*}
we obtain
\begin{equation*}
    \hat{\tau}_1 \geq \left( \frac{64c(c+\kappa)V_{\max}^2}{\kappa^2(1-\kappa)^2\left( \frac{\Delta}{2+\Delta} \right)^2\xi^2 \sigma^2\epsilon^2}\ln \frac{12cm^*|\mcs||\mca|(2c+\kappa)}{\kappa^2\delta} \right)^{\frac{1}{\omega}}.
\end{equation*}

Considering the conditions on $\hat{\tau}_1$ in~\Cref{lem:Gq} and~\Cref{lem:conditionalBound}, we choose 
\begin{equation*}
    \hat{\tau}_1 = \Theta\left( \left( \frac{V_{\max}^2}{(1-\gamma)^4\epsilon^2}\ln \frac{m^*|\mcs||\mca|V_{\max}^2}{(1-\gamma)^4\epsilon^2\delta} \right)^{\frac{1}{\omega}} \right).
\end{equation*}

Finally, applying the number of iterations $m^*=\left\lceil\frac{4}{1-\gamma}\ln \frac{2\gamma V_{\max}}{\epsilon(1-\gamma)}\right\rceil$ and Lemma \ref{lem:iteration}, we conclude that it suffices to let
\begin{align*}
    T&=\Omega\left( \left( \frac{V_{\max}^2}{(1-\gamma)^4\epsilon^2}\ln \frac{m^*|\mcs||\mca|V_{\max}^2}{(1-\gamma)^4\epsilon^2\delta} \right)^{\frac{1}{\omega}} + \left(\frac{2c}{\kappa}\frac{1}{1-\gamma} \ln\frac{\gamma V_{\max}}{(1-\gamma)\epsilon} \right)^{\frac{1}{1-\omega}} \right)\\
    &=\Omega\left( \left( \frac{V_{\max}^2}{(1-\gamma)^4\epsilon^2}\ln \frac{|\mcs||\mca|V_{\max}^2\ln(\frac{V_{\max}}{(1-\gamma)\epsilon})}{(1-\gamma)^5\epsilon^2\delta} \right)^{\frac{1}{\omega}} + \left(\frac{1}{1-\gamma} \ln\frac{ V_{\max}}{(1-\gamma)\epsilon} \right)^{\frac{1}{1-\omega}} \right)\\
    &=\Omega\left( \left( \frac{V_{\max}^2}{(1-\gamma)^4\epsilon^2}\ln \frac{|\mcs||\mca|V_{\max}^2}{(1-\gamma)^5\epsilon^2\delta} \right)^{\frac{1}{\omega}} + \left(\frac{1}{1-\gamma} \ln\frac{ V_{\max}}{(1-\gamma)\epsilon} \right)^{\frac{1}{1-\omega}} \right),
\end{align*}
to attain an $\epsilon$-accurate Q-estimator.

\input{SupplementaryBC}

%% file: SupplementaryBC.tex
\section{Proof of Theorem \ref{thm:asyncDQ} }\label{app:asyncThm}

The main idea of this proof is similar to that of Theorem \ref{thm:syncDQ} with further efforts to characterize the effects of asynchronous sampling. The proof also consists of three parts: (a) Part I which analyzes the stochastic error propagation between the two Q-estimators $\norm{Q^B_t  - Q^A_t }$; (b) Part II which analyzes the error dynamics between one Q-estimator and the optimum $\norm{Q^A_t -Q^* }$ conditioned on the error event in Part I; and (c) Part III which bounds the unconditional error $\norm{Q^A_t -Q^*}$.

To proceed the proof,  we first introduce the following notion of valid iterations for any fixed state-action pair $(s,a)$.

\begin{definition}\label{def:TAsa}
We define $T(s,a)$ as the collection of iterations if a state-action pair $(s,a)$ is used to update the Q-function $Q^A$ or $Q^B$, and $T^A(s, a)$ as the collection of iterations specifically updating $Q^A(s,a)$. In addition, we denote $T(s, a, t_1, t_2)$ and $T^A(s, a, t_1, t_2)$ as the set of iterations updating $(s,a)$ and $Q^A(s,a)$ between time $t_1$ and $t_2$, respectively. That is, 
\begin{align*}
    T(s, a, t_1, t_2) &= \left\{ t: t\in [t_1, t_2] \text{ and } t\in T(s,a) \right\},\\
    T^A(s, a, t_1, t_2) &= \left\{ t: t\in [t_1, t_2] \text{ and } t\in T^A(s,a) \right\}.
\end{align*}
Correspondingly, the number of iterations updating $(s,a)$ between time $t_1$ and $t_2$ equals the cardinality of $T(s, a, t_1, t_2)$ which is denoted as $|T(s, a, t_1, t_2)|$. Similarly, the number of iterations updating $Q^A(s,a)$ between time $t_1$ and $t_2$ is denoted as $|T^A(s, a, t_1, t_2)|$.
\end{definition}

Given Assumption \ref{asp:covering}, we can obtain some properties of the quantities defined above.
\begin{lemma}\label{prop:sa2L}
It always holds that $|T(s,a,t_1,t_2)|\leq t_2-t_1+1$ and $|T^A(s,a,t_1,t_2)|\leq t_2-t_1+1$. In addition, suppose that Assumption \ref{asp:covering} holds. Then we have $T(s,a,t,t+2kL-1)\geq k$ for any $t\geq 0$.
\end{lemma}
\begin{proof}
Since in a consecutive $2L$ running iterations of Algorithm \ref{alg:doubleQ}, either $Q^A$ or $Q^B$ is updated at least $L$ times. Then following from Assumption \ref{asp:covering}, $(s,a)$ is visited at least once for each $2L$ running iterations of Algorithm \ref{alg:doubleQ}, which immediately implies this proposition.
\end{proof}

Now we proceed our proof by three parts.

\subsection{Part I: Bounding $\norm{Q^B_t-Q^A_t}$} \label{subsec:PartIThm2}

We upper bound $\norm{Q^B_t - Q^A_t}$ block-wisely using a decreasing sequence $\{G_q\}_{q\geq0}$ as defined in~\Cref{lem:GqAsy} below.
\begin{proposition}\label{lem:GqAsy}
Fix $\epsilon>0, \kappa\in(\ln 2,1)$ and $\Delta\in(0, e^{\kappa}-2)$. Consider asynchronous double Q-learning using a polynomial learning rate $\alpha_t = \frac{1}{t^\omega}$ with $\omega\in(0,1)$. Suppose that Assumption \ref{asp:covering} holds. Let $G_q = (1-\xi)^q G_0$ with $G_0 = V_{\max}$ and $\xi=\frac{1-\gamma}{4}$. Let $\hat{\tau}_{q+1} = \hat{\tau}_q + \frac{2cL}{\kappa}\hat{\tau}_q^\omega$ for $q \geq 1$ with $c\geq \frac{L\kappa(\ln(2+\Delta) + 1/\tau_1^\omega)}{2(\kappa-\ln(2+\Delta) - 1/\tau_1^\omega)}$ and $\hat{\tau}_1$ as the finishing time of the first block satisfying 
\begin{equation*}
    \hat{\tau}_1\geq  \max\left\{\left(\frac{1}{\kappa-\ln(2+\Delta)}\right)^{\frac{1}{\omega}}, \left( \frac{128cL(cL+\kappa)V_{\max}^2}{\kappa^2\left(\frac{\Delta}{2+\Delta}\right)^2\xi^2\sigma^2\epsilon^2 }\ln\left(\frac{64cL(cL+\kappa)V_{\max}^2}{\kappa^2\left(\frac{\Delta}{2+\Delta}\right)^2\xi^2\sigma^2\epsilon^2 }\right) \right)^{\frac{1}{\omega}} \right\}.
\end{equation*}
Then for any $n$ such that $G_n\geq\sigma\epsilon$, we have
\begin{align*}
    &\mP\left[ \forall q\in [0,n], \forall t\in[\hat{\tau}_{q+1},\hat{\tau}_{q+2}), \norm{Q^B_t - Q^A_t}\leq G_{q+1} \right]\\
    &\quad\geq 1- \frac{4cL(n+1)}{\kappa}\left(1+\frac{2cL}{\kappa}\right)|\mcs||\mca| \exp\left( -\frac{\kappa^2\left( \frac{\Delta}{2+\Delta} \right)^2\xi^2 \sigma^2\epsilon^2\hat{\tau}_1^{\omega}}{64cL(cL+\kappa)V_{\max}^2} \right).
\end{align*}
\end{proposition}
The proof of~\Cref{lem:GqAsy} consists of the following steps. Since the main idea of the proofs is similar to that of~\Cref{lem:Gq}, we will focus on pointing out the difference. We continue to use the notation $u^{BA}_t(s,a):=Q^B_t(s,a)-Q^A_t(s,a)$.

\textbf{Step 1: Characterizing the dynamics of $u^{BA}_t$}

First, we observe that when $(s,a)$ is visited at time $t$, i.e., $t\in T(s,a)$, Lemmas \ref{lem:uBAdyn} and \ref{lem:uBAsanwich} still apply. Otherwise, $u^{BA}$ is not updated. Thus, we have
\begin{equation*}
    u_{t+1}^{BA}(s,a) = \left\{\begin{aligned} 
    &u_t^{BA}(s,a),\quad t\notin T(s,a);\\
    &(1-\alpha_t)u_t^{BA}(s,a) + \alpha_t F_t(s,a),\quad t\in T(s,a),
    \end{aligned}\right.
\end{equation*}
where $F_t$ satisfies
\begin{equation*}
    \norm{\mE[F_t|\mcF_t]} \leq \frac{1+\gamma}{2} \norm{u_t^{BA}}.
\end{equation*}
For $t\in T(s,a)$, we rewrite the dynamics of $u_t^{BA}(s,a)$ as
\begin{equation*}
    u_{t+1}^{BA}(s,a) = (1-\alpha_t)u_t^{BA}(s,a) + \alpha_t F_t = (1-\alpha_t)u_t^{BA}(s,a) + \alpha_t (h_t(s,a) + z_t(s,a)),
\end{equation*}
where $h_t(s,a) = \mE[F_t(s,a)|\mcF_t]$ and $z_t(s,a) = F_t(s,a) - \mE[F_t(s,a)|\mcF_t]$. 

In the following steps, we use induction to proceed the proof of~\Cref{lem:GqAsy}. Given $G_q$ defined in~\Cref{lem:GqAsy}, since $\norm{u_t^{BA} }\leq G_0$ holds for all $t$, and thus it holds for $t\in [0,\hat{\tau}_1]$. Now suppose $\hat{\tau}_q$ satisfies that $\norm{u_t^{BA} }\leq G_q$ for any $t\geq \hat{\tau}_q$. Then we will show there exists $\hat{\tau}_{q+1} = \hat{\tau}_q + \frac{2cL}{\kappa}\hat{\tau}_q^\omega$ such that $\norm{u_t^{BA} }\leq G_{q+1}$ for any $t\geq \hat{\tau}_{q+1}$.

\textbf{Step 2: Constructing sandwich bounds}

We first observe that the following sandwich bound still holds for all $t\geq \hat{\tau}_q$.
\begin{equation*}
    -X_{t;\hat{\tau}_q}(s,a) + Z_{t;\hat{\tau}_q}(s,a) \leq u_t^{BA}(s,a) \leq X_{t;\hat{\tau}_q}(s,a) + Z_{t;\hat{\tau}_q}(s,a),
\end{equation*}
where $Z_{t;\hat{\tau}_q}(s,a)$ is defined as
\begin{equation*}
    Z_{t+1;\hat{\tau}_q}(s,a) = \left\{\begin{aligned}
    & Z_{t;\hat{\tau}_q}(s,a), \quad t\notin T(s,a)\\
    & (1-\alpha_t)Z_{t;\hat{\tau}_q}(s,a) + \alpha_t z_t(s,a), \quad t\in T(s,a),
    \end{aligned}
    \right.
\end{equation*}
with the initial condition $Z_{\hat{\tau}_q;\hat{\tau}_q}(s,a) = 0$, and $X_{t;\hat{\tau}_q}(s,a)$ is defined as
\begin{equation*}
    X_{t+1;\hat{\tau}_q}(s,a) = \left\{\begin{aligned}
    & X_{t;\hat{\tau}_q}(s,a), \quad t\notin T(s,a)\\
    & (1-\alpha_t)X_{t;\hat{\tau}_q}(s,a) + \alpha_t \gamma'G_q, \quad t\in T(s,a),
    \end{aligned}
    \right.
\end{equation*}
with $X_{\hat{\tau}_q;\hat{\tau}_q}(s,a) = G_q,\gamma'=\frac{1+\gamma}{2}$.

This claim can be shown by induction. This bound clearly holds for the initial case with $t=\hat{\tau}_q$. Assume that it still holds for iteration $t$. If $t\in T(s,a)$, the proof is the same as that of~\Cref{lem:uBAsanwich}. If $t\notin T(s,a)$, since all three sequences do not change from time $t$ to time $t+1$, the sandwich bound still holds. Thus we conclude this claim.

\textbf{Step 3: Bounding $X_{t;\hat{\tau}_q}(s,a)$}

Next, we bound the deterministic sequence $X_{t;\hat{\tau}_q}(s,a)$. Observe that $X_{t;\hat{\tau}_q}(s,a)\leq G_q$ for any $t\geq\hat{\tau}_q$. We will next show that $X_{t;\hat{\tau}_q}(s,a) \leq \left(\gamma' + \frac{2}{2+\Delta}\xi\right)G_q$ for any $t\in [\hat{\tau}_{q+1},\hat{\tau}_{q+2})$ where $\hat{\tau}_{q+1} = \hat{\tau}_q + \frac{2cL}{\kappa}\hat{\tau}_q^\omega$. 

Similarly to the proof of~\Cref{lem:Xt}, we still rewrite $X_{\hat{\tau}_q;\hat{\tau}_q}(s,a)$ as $X_{\hat{\tau}_q;\hat{\tau}_q}(s,a) = G_q = \gamma' G_q + (1-\gamma')G_q := \gamma' G_q + \rho_{\hat{\tau}_q} $. However, in this case the dynamics of $X_{t;\hat{\tau}_q}(s,a)$ is different, which is represented as
\begin{equation*}
    X_{t+1;\hat{\tau}_q}(s,a) =\left\{\begin{aligned} 
    & X_{t;\hat{\tau}_q}(s,a), \quad t\notin T(s,a)\\
    &(1-\alpha_t)X_{t;\hat{\tau}_q}(s,a) + \alpha_t \gamma'G_q = \gamma'G_q + (1-\alpha_t)\rho_t, \quad t\in T(s,a).
    \end{aligned}
    \right.
\end{equation*}
where $\rho_{t+1} = (1-\alpha_t)\rho_t$ when $t\in T(s,a)$. By the definition of $\rho_t$, we obtain
\begin{align*}
    \rho_t &= \rho_{\hat{\tau}_q}\underset{i\in T(s, a, \hat{\tau}_q, t-1)}{\Pi}(1-\alpha_i) 
    = (1-\gamma')G_q\underset{i\in T(s, a, \hat{\tau}_q, t-1)}{\Pi}(1-\alpha_i)\\
    &\leq (1-\gamma')G_q\underset{i\in T(s, a, \hat{\tau}_q, \hat{\tau}_{q+1}-1)}{\Pi}\left(1-\frac{1}{i^\omega}\right) 
    \leq (1-\gamma')G_q\prod_{i=\hat{\tau}_{q+1}-|T(s, a, \hat{\tau}_q, \hat{\tau}_{q+1}-1)|}^{\hat{\tau}_{q+1}-1}\left(1-\frac{1}{i^\omega}\right)\\
    &\overset{\text{(i)}}{\leq} (1-\gamma')G_q\prod_{i=\hat{\tau}_{q+1}-\frac{c}{\kappa}\hat{\tau}_q^\omega}^{\hat{\tau}_{q+1}-1}\left(1-\frac{1}{i^\omega}\right)
    \overset{\text{(ii)}}{\leq} (1-\gamma')G_q\exp\left( -\frac{\frac{c}{\kappa}\hat{\tau}_q^\omega-1}{(\hat{\tau}_{q+1}-1)^\omega} \right)\\
    &\leq (1-\gamma')G_q\exp\left( -\frac{\frac{c}{\kappa}\hat{\tau}_q^\omega-1 }{\hat{\tau}_{q+1}^\omega} \right)
    = (1-\gamma')G_q\exp\left( -\frac{c}{\kappa}\left(\frac{\hat{\tau}_q}{\hat{\tau}_{q+1}}\right)^\omega + \frac{1}{\hat{\tau}_{q+1}^\omega} \right)\\
    &\overset{\text{(iii)}}{\leq} (1-\gamma')G_q\exp\left( -\frac{c}{\kappa}\frac{1}{1+\frac{2cL}{\kappa}} + \frac{1}{\hat{\tau}_{1}^\omega} \right),
\end{align*}
where (i) follows from~\Cref{prop:sa2L}, (ii) follows~\Cref{lem:prodHelp}, and (iii) follows because $\hat{\tau}_q\geq\hat{\tau}_1$ and
\begin{equation*}
    \left(\frac{\hat{\tau}_q}{\hat{\tau}_{q+1}}\right)^\omega \geq \frac{\hat{\tau}_q}{\hat{\tau}_{q+1}} = \frac{\hat{\tau}_q}{\hat{\tau}_{q} + \frac{2cL}{\kappa}\hat{\tau}_q^\omega}\geq \frac{1}{1+\frac{2cL}{\kappa}}.
\end{equation*}

Since $\kappa\in(\ln 2,1)$ and $\Delta\in(0, e^{\kappa}-2)$, we have $\ln (2+\Delta)\in (0, \kappa)$. Further, observing $\hat{\tau_1}^\omega > \frac{1}{\kappa-\ln(2+\Delta)}$, we obtain $\ln(2+\Delta) + \frac{1}{\hat{\tau_1}^\omega} \in (0,\kappa)$. Last, since $c\geq\frac{L\kappa(\ln(2+\Delta) + 1/\hat{\tau_1}^\omega)}{2(\kappa-\ln(2+\Delta) - 1/\hat{\tau_1}^\omega)}$, we have $-\frac{c}{1+\frac{2c}{\kappa}} + \frac{1}{\hat{\tau_1}^\omega}\leq -\ln(2+\Delta)$.

Finally, combining the above observations with the fact $1-\gamma'=2\xi$, we conclude that for any $t\geq \hat{\tau}_{q+1} = \hat{\tau}_q + \frac{2cL}{\kappa}\hat{\tau}_q^\omega$,
\begin{equation*}
    X_{t;\hat{\tau}_q}(s,a) \leq \left(\gamma' + \frac{2}{2+\Delta}\xi\right)G_q.
\end{equation*}

\textbf{Step 4: Bounding $Z_{t;\hat{\tau}_q}(s,a)$}

It remains to bound the stochastic sequence $Z_{t;\hat{\tau}_q}(s,a)$ by $\frac{\Delta}{2+\Delta}\xi G_q$ at epoch $q+1$. We define an auxiliary sequence $\{Z_{t;\hat{\tau}_q}^l (s,a)\}$ (which is different from that in \eqref{eq:Zl}) as:
\begin{equation*}
    Z^l_{t;\hat{\tau}_q}(s,a) = \sum_{i\in T(s, a, \hat{\tau}_q, t-1)}\alpha_i\underset{j\in T(s, a, i+1, t-1)}{\Pi}(1-\alpha_j) z_i(s,a).
\end{equation*}
Following the same arguments as the proof of Lemma \ref{lem:ZlDiff}, we conclude that $\{Z_{t;\hat{\tau}_q}^l (s,a)\}$ is a martingale sequence and satisfies
\begin{equation*}
    \lvert Z^l_{t;\hat{\tau}_q}(s,a) - Z^{l-1}_{t;\hat{\tau}_q}(s,a) \rvert = \alpha_{\hat{\tau}_q+l} |z_{\hat{\tau}_q + l}(s,a)|\leq \frac{2V_{\max}}{\hat{\tau}_q^\omega}.
\end{equation*}

In addition, note that 
\begin{align*}
    Z_{t;\hat{\tau}_q}(s,a) &= Z_{t;\hat{\tau}_q}(s,a) - Z_{\hat{\tau}_q;\hat{\tau}_q}(s,a)\\
    &= \sum_{l:\hat{\tau}_q+l-1\in T(s, a, \hat{\tau}_q, t-1)} (Z^l_{t;\hat{\tau}_q}(s,a) - Z^{l-1}_{t;\hat{\tau}_q}(s,a)) + Z^{\min(T(s, a, \hat{\tau}_q, t-1))}_{t;\hat{\tau}_q}(s,a).
\end{align*}
Then we apply Azuma' inequality in Lemma \ref{lem:azuma} and obtain
\begin{align*}
    &\mP\left[ \lvert Z_{t;\hat{\tau}_q}(s,a)  \rvert > \hat{\epsilon}| t\in[\hat{\tau}_{q+1},\hat{\tau}_{q+2})  \right]\\
    &\quad\leq 2\exp\left( \frac{-\hat{\epsilon}^2}{2\underset{l:\hat{\tau}_q+l-1\in T(s, a, \hat{\tau}_q, t-1)}{\sum} (Z^l_{t;\hat{\tau}_q}(s,a) \!-\! Z^{l-1}_{t;\hat{\tau}_q}(s,a))^2 \!+\! 2\left(Z^{\min(T(s, a, \hat{\tau}_q, t-1))}_{t;\hat{\tau}_q}(s,a)\right)^2} \right)\\
    &\quad\leq 2\exp\left( -\frac{\hat{\epsilon}^2\hat{\tau}_q^{2\omega}}{8(|T(s, a, \hat{\tau}_q, t-1)|+1) V_{\max}^2} \right) 
    \overset{\text{(i)}}{\leq} 2\exp\left( -\frac{\hat{\epsilon}^2\hat{\tau}_q^{2\omega}}{8(t-\hat{\tau}_q)  V_{\max}^2} \right)\\
    &\quad\leq 2\exp\left( -\frac{\hat{\epsilon}^2\hat{\tau}_q^{2\omega}}{8(\hat{\tau}_{q+2}-\hat{\tau}_q)  V_{\max}^2} \right) 
    = 2\exp\left( -\frac{\hat{\epsilon}^2\hat{\tau}_q^{2\omega}}{8 \left(\frac{2cL}{\kappa}\hat{\tau}_{q+1}^\omega+\frac{2cL}{\kappa}\hat{\tau}_q^\omega\right) V_{\max}^2} \right)\\
    &\quad = 2\exp\left( -\frac{\hat{\epsilon}^2\hat{\tau}_q^{2\omega}}{8 \left(\frac{2cL}{\kappa}(\hat{\tau}_{q}+\frac{2cL}{\kappa}\hat{\tau}_q^\omega)^\omega+\frac{2cL}{\kappa}\hat{\tau}_q^\omega\right) V_{\max}^2} \right)\\
    &\quad\leq 2\exp\left( -\frac{\kappa^2\hat{\epsilon}^2\hat{\tau}_q^{\omega}}{32cL(cL+\kappa) V_{\max}^2} \right)
\end{align*}
where (i) follows from~\Cref{prop:sa2L}.

\textbf{Step 5: Taking union over all blocks}

Finally, using the union bound of~\Cref{lem:unionBound} yields
\begin{align*}
    &\mP\left[ \forall q\in [0,n], \forall t\in[\hat{\tau}_{q+1},\hat{\tau}_{q+2}), \norm{Q^B_t  - Q^A_t }\leq G_{q+1} \right]\\
    &\quad\geq \mP\left[ \forall(s,a), \forall q\in [0,n], \forall t\in[\hat{\tau}_{q+1},\hat{\tau}_{q+2}), \lvert Z_{t;\hat{\tau}_q}(s,a)  \rvert \leq \frac{\Delta}{2+\Delta}\xi G_q \right]\\
    &\quad\geq 1 - \sum_{q=0}^n |\mcs||\mca|(\hat{\tau}_{q+2} - \hat{\tau}_{q+1}) \cdot \mP\left[ \lvert Z_{t;\hat{\tau}_q}(s,a)  \rvert > \frac{\Delta}{2+\Delta}\xi G_q \Big\rvert t\in[\hat{\tau}_{q+1},\hat{\tau}_{q+2})  \right]\\
    &\quad\geq 1 - \sum_{q=0}^n |\mcs||\mca| \frac{2cL}{\kappa}\hat{\tau}_{q+1}^\omega \cdot 2\exp\left( -\frac{\kappa^2\left( \frac{\Delta}{2+\Delta} \right)^2\xi^2 G_q^2\hat{\tau}_q^{\omega}}{32cL(cL+\kappa)V_{\max}^2} \right)\\
    &\quad\geq 1 - \sum_{q=0}^n |\mcs||\mca| \frac{2cL}{\kappa}\left(1+\frac{2cL}{\kappa}\right)\hat{\tau}_{q}^\omega \cdot 2\exp\left( -\frac{\kappa^2\left( \frac{\Delta}{2+\Delta} \right)^2\xi^2 G_q^2\hat{\tau}_q^{\omega}}{32cL(cL+\kappa)V_{\max}^2} \right)\\
    &\quad\overset{\text{(i)}}{\geq} 1 - \sum_{q=0}^n |\mcs||\mca| \frac{2cL}{\kappa}\left(1+\frac{2cL}{\kappa}\right)\hat{\tau}_{q}^\omega \cdot 2\exp\left( -\frac{\kappa^2\left( \frac{\Delta}{2+\Delta} \right)^2\xi^2 \sigma^2\epsilon^2\hat{\tau}_q^{\omega}}{32cL(cL+\kappa)V_{\max}^2} \right)\\
    &\quad\overset{\text{(ii)}}{\geq} 1 - \frac{4cL}{\kappa}\left(1+\frac{2cL}{\kappa}\right)\sum_{q=0}^n |\mcs||\mca| \cdot \exp\left( -\frac{\kappa^2\left( \frac{\Delta}{2+\Delta} \right)^2\xi^2 
    \sigma^2\epsilon^2\hat{\tau}_q^{\omega}}{64cL(cL+\kappa)V_{\max}^2} \right)\\
    &\quad\overset{\text{(iii)}}{\geq} 1- \frac{4cL(n+1)}{\kappa}\left(1+\frac{2cL}{\kappa}\right)|\mcs||\mca| \exp\left( -\frac{\kappa^2\left( \frac{\Delta}{2+\Delta} \right)^2\xi^2 \sigma^2\epsilon^2\hat{\tau}_1^{\omega}}{64cL(cL+\kappa)V_{\max}^2} \right),
\end{align*}
where (i) follows from $G_q \geq G_n \geq \sigma\epsilon $, (ii) follows from~\Cref{lem:tauHelp} by substituting $a=\frac{64cL(cL+\kappa)V_{\max}^2}{\kappa^2\left(\frac{\Delta}{2+\Delta}\right)^2\xi^2\sigma^2\epsilon^2 }, b=1$ and observing that
\begin{align*}
    \hat{\tau}_q^{\omega}&\geq\hat{\tau}_1^{\omega}\geq \frac{128cL(cL+\kappa)V_{\max}^2}{\kappa^2\left(\frac{\Delta}{2+\Delta}\right)^2\xi^2\sigma^2\epsilon^2 }\ln\left(\frac{64cL(cL+\kappa)V_{\max}^2}{\kappa^2\left(\frac{\Delta}{2+\Delta}\right)^2\xi^2\sigma^2\epsilon^2 }\right)=2ab\ln ab,
\end{align*}
and (iii) follows from $\hat{\tau}_q \geq \hat{\tau}_1$.

\subsection{Part II: Conditionally bounding $\norm{Q^A_t-Q^*}$} \label{subsec:PartIIThm2}
We upper bound $\norm{Q^A_t-Q^*}$ block-wisely by a decreasing sequence $\{D_k\}_{k\geq0}$ conditioned on the following two events: fix a positive integer $m$, 
\begin{align}
    G & = \left\{ \forall(s,a), \forall k\in [0,m], \forall t\in[\tau_{k+1},\tau_{k+2}), \norm{Q^B_t  - Q^A_t }\leq \sigma D_{k+1} \right\}, \label{eq:eventE} \\
    H & = \{ \forall k\in [1,m+1], I^A_k\geq cL\tau_{k}^\omega \}, \label{eq:eventF}
\end{align}
where $I^A_k$ denotes the number of iterations updating $Q^A$ at epoch $k$, $\tau_k$ is the starting iteration index of the $k+1$th block, and $\omega$ is the parameter of the polynomial learning rate. Roughly, Event $G$ requires that the difference between the two Q-function estimators are bounded appropriately, and Event $H$ requires that $Q^A$ is sufficiently updated in each epoch. Again, we will design $\{D_k\}_{k\geq 0}$ in a way such that the occurrence of Event $G$ can be implied from the event that $\norm{u^{BA}_t}$ is bounded by $\{G_q\}_{q\geq 0}$ (see~\Cref{lem:coupleAsy} below). A lower bound of the probability for Event $H$ to hold is characterized in~\Cref{lem:halfQA} in Part III.
\begin{proposition}\label{lem:conditionalBoundAsy}
Fix $\epsilon>0, \kappa\in(\ln 2,1)$ and $\Delta\in(0, e^{\kappa}-2)$. Consider asynchronous double Q-learning using a polynomial learning rate $\alpha_t = \frac{1}{t^\omega}$ with $\omega\in(0,1)$. Let $\{G_q\}, \{\hat{\tau}_q\}$ be as defined in~\Cref{lem:GqAsy}. Define $D_k = (1-\beta)^k\frac{V_{\max}}{\sigma}$ with $\beta = \frac{1-\gamma(1+\sigma)}{2}$ and $\sigma = \frac{1-\gamma}{2\gamma}$. Let $\tau_k=\hat{\tau}_k$ for $k\geq0$.  Suppose that $c\geq \frac{L(\ln(2+\Delta) + 1/\tau_1^\omega)}{2(\kappa-\ln(2+\Delta) - 1/\tau_1^\omega)}$ and $\tau_1 = \hat{\tau}_1$ as the finishing time of the first epoch satisfies 
\begin{equation*}
    \tau_1\geq  \max\left\{\left(\frac{1}{\kappa-\ln(2+\Delta)}\right)^{\frac{1}{\omega}}, \left( \frac{32cL(cL+\kappa)V_{\max}^2}{\kappa^2\left(\frac{\Delta}{2+\Delta}\right)^2\beta^2\epsilon^2 }\ln \left(\frac{16cL(cL+\kappa)V_{\max}^2}{\kappa^2\left(\frac{\Delta}{2+\Delta}\right)^2\beta^2\epsilon^2 }\right) \right)^{\frac{1}{\omega}} \right\}.
\end{equation*}
Then for any $m$ such that $D_m\geq\epsilon$, we have
\begin{align*}
    &\mP\left[ \forall k\in [0,m], \forall t\in[\tau_{k+1},\tau_{k+2}), \norm{Q^A_t  - Q^* }\leq D_{k+1} |G,H \right]\\
    &\quad\geq 1 - \frac{4cL(m+1)}{\kappa}\left(1+\frac{2cL}{\kappa}\right)|\mcs||\mca| \exp\left( -\frac{\kappa^2\left( 1-\frac{2}{e} \right)^2\beta^2 \epsilon^2\tau_1^{\omega}}{16cL(cL+\kappa)V_{\max}^2} \right).
\end{align*}
\end{proposition}
Recall that in the proof of~\Cref{lem:conditionalBound}, $Q^A$ is not updated at each iteration and thus we introduced notations $T^A$ and $T^A(t_1,t_2)$ in Definition \ref{def:TA} to capture the convergence of the error $\norm{Q^A -Q^* }$. In this proof, the only difference is that when choosing to update $Q^A$, only one $(s,a)$-pair is visited. Therefore, the proof of~\Cref{lem:conditionalBoundAsy} is similar to that of~\Cref{lem:conditionalBound}, where most of the arguments simply substitute $T^A, T^A(t_1,t_2)$ in the proof of~\Cref{lem:conditionalBound} by $T^A(s,a), T^A(s,a,t_1,t_2)$ in Definition \ref{def:TAsa}, respectively. Certain bounds are affected by such substitutions. In the following, we proceed the proof of~\Cref{lem:conditionalBoundAsy} in five steps, and focus on pointing out the difference from the proof of~\Cref{lem:conditionalBound}. More details can be referred to~\Cref{subsec:PartII}.

\textbf{Step 1: Coupling $\{D_k\}_{k\geq0}$ and $\{G_q\}_{q\geq0}$}

We establish the relationship between $\{D_k\}_{k\geq0}$ and $\{G_q\}_{q\geq0}$ in the same way as Lemma \ref{lem:couple}. For the convenience of reference, we restate \Cref{lem:couple} in the following.
\begin{lemma}\label{lem:coupleAsy}
Let $\{G_q\}$ be defined in~\Cref{lem:GqAsy}, and let $D_k = (1-\beta)^k\frac{V_{\max}}{\sigma}$ with $\beta = \frac{1-\gamma(1+\sigma)}{2}$ and $\sigma = \frac{1-\gamma}{2\gamma}$. Then we have
\begin{align*}
    &\mP\left[ \forall(s,a), \forall q\in [0,m], \forall t\in[\hat{\tau}_{q+1},\hat{\tau}_{q+2}), \norm{Q^B_t  - Q^A_t }\leq G_{q+1} \right]\\
    &\quad\leq \mP\left[ \forall(s,a), \forall k\in [0,m], \forall t\in[\tau_{k+1},\tau_{k+2}), \norm{Q^B_t  - Q^A_t }\leq \sigma D_{k+1} \right],
\end{align*}
given that $\tau_k = \hat{\tau}_{k }$.
\end{lemma}

\textbf{Step 2: Constructing sandwich bounds}

Let $r_t(s,a)=Q^A(s,a)-Q^*(s,a)$ and $\tau_k$ be such that $\norm{r_t}\leq D_k$ for all $t\geq\tau_k$. The requirement of Event $G$ yields
\begin{equation*}
    -Y_{t;\tau_k}(s,a) + W_{t;\tau_k}(s,a) \leq r_t(s,a) \leq Y_{t;\tau_k}(s,a) + W_{t;\tau_k}(s,a),
\end{equation*}
where $W_{t;\tau_k}(s,a)$ is defined as
\begin{equation*}
    W_{t+1;\tau_k}(s,a) = \left\{\begin{aligned}
    &W_{t;\tau_k}(s,a), \quad t\notin T^A(s,a);\\
    &(1-\alpha_t)W_{t;\tau_k}(s,a) + \alpha_t w_t(s,a), \quad t\in T^A(s,a),
    \end{aligned}\right.
\end{equation*}
with $w_{t}(s,a) = \mcT_{t} Q_{t}^A(s,a) - \mcT Q_{t}^A(s,a)$ and $W_{\tau_k;\tau_k}(s,a) = 0$, and $Y_{t;\tau_k}(s,a)$ is given by
\begin{equation*}
    Y_{t+1;\tau_k}(s,a) = \left\{\begin{aligned}
    &Y_{t;\tau_k}(s,a), \quad t\notin T^A(s,a);\\
    &(1-\alpha_t)Y_{t;\tau_k}(s,a) + \alpha_t \gamma''D_k, \quad t\in T^A(s,a),
    \end{aligned}\right.
\end{equation*}
with $Y_{\tau_k;\tau_k}(s,a) = D_k$ and $\gamma''=\gamma(1+\sigma)$.

\textbf{Step 3: Bounding $Y_{t;\tau_k}(s,a)$}

Next, we first bound $Y_{t;\tau_k}(s,a)$. Observe that $Y_{t;\tau_k}(s,a) \leq D_k$ for any $t \geq \tau_k$. We will bound $Y_{t;\tau_k}(s,a)$ by $\left(\gamma'' + \frac{2}{2+\Delta}\beta\right)D_k$ for block $k+1$.

We use a similar representation of $Y_{t;\tau_k}(s,a)$ as in the proof of Lemma \ref{lem:Yt}, which is given by
\begin{equation*}
    Y_{t+1;\tau_k}(s,a) = \left\{\begin{aligned}
    & Y_{t;\tau_k}(s,a), \quad t\notin T^A(s,a)\\
    &(1-\alpha_t)Y_{t;\tau_k}(s,a) + \alpha_t \gamma''G_q = \gamma''G_q + (1-\alpha_t)\rho_t, \quad t\in T^A(s,a)
    \end{aligned}\right.
\end{equation*}
where $\rho_{t+1} = (1-\alpha_t)\rho_t$ for $t\in T^A(s,a)$. By the definition of $\rho_t$, we obtain
\begin{align*}
    \rho_t &= \rho_{\tau_k}\prod_{i\in T^A(s, a, \tau_k, t-1)}(1-\alpha_i) = (1-\gamma'')D_k\prod_{i\in T^A(s, a, \tau_k, t-1)}(1-\alpha_i)\\
    &= (1-\gamma'')D_k\prod_{i\in T^A(s, a, \tau_k, t-1)}\left(1-\frac{1}{i^\omega}\right) 
    \overset{\text{(i)}}{\leq} (1-\gamma'')D_k\prod_{i\in T^A(s, a, \tau_k, \tau_{k+1}-1)}\left(1-\frac{1}{i^\omega}\right)\\
    &\overset{\text{(ii)}}{\leq} (1-\gamma'')D_k\prod_{i=\tau_{k+1}-c\tau_k^\omega}^{\tau_{k+1}-1}\left(1-\frac{1}{i^\omega}\right)
    \overset{\text{(iii)}}{\leq} (1-\gamma'')D_k\exp\left( -\frac{c\tau_k^\omega-1}{(\tau_{k+1}-1)^\omega} \right) \nonumber\\
    &\leq (1-\gamma'')D_k\exp\left( -\frac{c\tau_k^\omega-1}{\tau_{k+1}^\omega} \right) 
    = (1-\gamma'')D_k\exp\left( -c\left(\frac{\tau_k}{\tau_{k+1}}\right)^\omega + \frac{1}{\tau_{k+1}^\omega} \right) \nonumber\\
    &\overset{\text{(iv)}}{\leq} (1-\gamma'')D_k\exp\left( -\frac{c}{1+\frac{2Lc}{\kappa}} + \frac{1}{\tau_{1}^\omega} \right), \nonumber
\end{align*}
where (i) follows because $\alpha_i<1$ and $t\geq \tau_{k+1}$, (ii) follows from Proposition \ref{prop:sa2L} and the requirement of event $H$, (iii) follows from Lemma \ref{lem:tauHelp}, and (iv) holds because $\tau+k\geq\tau_1$ and 
\begin{equation*}
    \left(\frac{\tau_k}{\tau_{k+1}}\right)^\omega \geq \frac{\tau_k}{\tau_{k+1}} = \frac{\tau_k}{\tau_{k} + \frac{2cL}{\kappa}\tau_k^\omega}\geq \frac{1}{1+\frac{2cL}{\kappa}}.
\end{equation*}
Since $\kappa\in(\ln 2,1)$ and $\Delta\in(0, e^{\kappa}-2)$, we have $\ln (2+\Delta)\in (0, \kappa)$. Further, observing $\hat{\tau_1}^\omega > \frac{1}{\kappa-\ln(2+\Delta)}$, we obtain $\ln(2+\Delta) + \frac{1}{\hat{\tau_1}^\omega} \in (0,\kappa)$. Last, since $c\geq\frac{L(\ln(2+\Delta) + 1/\hat{\tau_1}^\omega)}{2(\kappa-\ln(2+\Delta) - 1/\hat{\tau_1}^\omega)}$, we have $-\frac{c}{1+\frac{2c}{\kappa}} + \frac{1}{\hat{\tau_1}^\omega}\leq -\ln(2+\Delta)$.

Then, we have $\rho_t\leq \frac{1-\gamma''}{2+\Delta}D_k$. Thus we conclude that for any $t\in [\tau_{k+1}, \tau_{k+2}]$,
\begin{equation*}
    Y_{t;\tau_k}(s,a) \leq \left(\gamma'' + \frac{2}{2+\Delta}\beta\right)D_k.
\end{equation*}

\textbf{Step 4: Bounding $W_{t;\tau_k}(s,a)$}

It remains to bound $|W_{t;\tau_k}(s,a)|\leq \left(1-\frac{2}{2+\Delta}\right)\beta D_k$ for $t\in [\tau_{k+1},\tau_{k+2})$.

Similarly to \Cref{subsec:proofProp2}, we define a new sequence $\{W_{t;\tau_k}^l(s,a)\}$ as
\begin{equation*}
    W^l_{t;\tau_k}(s,a) = \sum_{i\in T^A(s, a, \tau_k, \tau_k+l)} \alpha_i\underset{j\in T^A(s, a, i+1, t-1)}{\Pi} (1-\alpha_j)w_i(s,a).
\end{equation*}
The same arguments as the proof of Lemma \ref{lem:WlDiff} yields
\begin{equation*}
    \lvert W^l_{t;\tau_k}(s,a) - W^{l-1}_{t;\tau_k}(s,a) \rvert \leq  \frac{V_{\max}}{\tau_k^\omega}.
\end{equation*}

If we fix $\tilde{\epsilon}>0$, then for any $t\in[\tau_{k+1},\tau_{k+2})$ we have
\begin{align*}
    &\mP\left[ |W_{t;\tau_k}(s,a)|>\tilde{\epsilon} | t\in[\tau_{k+1},\tau_{k+2}),G,H \right]\\
    &\leq 2\exp\left( \frac{-\tilde{\epsilon}^2}{2\underset{l:\tau_k+l-1\in T^A(s,a,\tau_k, t\!-\!1)}{\sum}\!\left( W^l_{t;\tau_k}(s,a) \!-\! W^{l-1}_{t;\tau_k}(s,a) \right)^2 \!+\! 2(W^{\min(T^A(s,a,\tau_k, t\!-\!1))}_{t;\tau_k}(s,a))^2  } \right)\\
    &\leq 2\exp\left( -\frac{\hat{\epsilon}^2\tau_k^{2\omega}}{2(|T^A(s,a,\tau_k,t-1)|+1)V_{\max}^2} \right) \overset{\text{(i)}}{\leq} 2\exp\left( -\frac{\tilde{\epsilon}^2\tau_k^{2\omega}}{2(t-\tau_k)V_{\max}^2} \right)\\
    &\leq 2\exp\left( -\frac{\tilde{\epsilon}^2\tau_k^{2\omega}}{2(\tau_{k+2}-\tau_k)V_{\max}^2} \right) 
    \overset{\text{(ii)}}{\leq} 2\exp\left( -\frac{\kappa^2\tilde{\epsilon}^2\tau_k^{\omega}}{8cL(cL+\kappa)V_{\max}^2} \right)\\
    &= 2\exp\left( -\frac{\kappa^2\tilde{\epsilon}^2\tau_k^{\omega}}{8cL(cL+\kappa)V_{\max}^2} \right),
\end{align*}
where (i) follows from Proposition \ref{prop:sa2L} and (ii) holds because
\begin{equation*}
    \tau_{k+2} - \tau_k = \frac{2cL}{\kappa}\tau_{k+1}^\omega + \frac{2cL}{\kappa}\tau_k^\omega = \frac{2cL}{\kappa}\left( \tau_k + \frac{2cL}{\kappa}\tau_k^\omega \right)^\omega + \frac{2cL}{\kappa}\tau_k^\omega \leq \frac{4cL(cL+\kappa)}{\kappa^2}\tau_k^\omega.
\end{equation*}

\textbf{Step 5: Taking union over all blocks }

Applying the union bound in Lemma \ref{lem:unionBound}, we obtain
\begin{align*}
    &\mP\left[ \forall k\in [0,m], \forall t\in[\tau_{k+1},\tau_{k+2}), \norm{Q^A_t  - Q^* }\leq D_{k+1}  | G,H\right]\\
    &\quad\geq \mP\left[ \forall(s,a), \forall k\in [0,m], \forall t\in[\tau_{k+1},\tau_{k+2}), \lvert W_{t;\tau_k}(s,a)  \rvert \leq \frac{\Delta}{2+\Delta}\beta D_k|G,H \right]\\
    &\quad\geq 1 - \sum_{k=0}^m |\mcs||\mca|(\tau_{k+2}-\tau_{k+1}) \cdot \mP\left[ \lvert W_{t;\tau_k}(s,a)  \rvert > \frac{\Delta}{2+\Delta}\beta D_k \Big\rvert t\in[\tau_{k+1},\tau_{k+2}),G,H  \right]\\
    &\quad\geq 1 - \sum_{k=0}^m |\mcs||\mca| \frac{2cL}{\kappa}\tau_{k+1}^\omega \cdot 2\exp\left( -\frac{\kappa^2\left( \frac{\Delta}{2+\Delta} \right)^2\beta^2 D_k^2\tau_k^{\omega}}{8cL(cL+\kappa)V_{\max}^2} \right)\\
    &\quad\geq 1 - \sum_{k=0}^m |\mcs||\mca| \frac{2cL}{\kappa}\left(1+\frac{2cL}{\kappa}\right)\tau_{k}^\omega \cdot 2\exp\left( -\frac{\kappa^2\left( \frac{\Delta}{2+\Delta} \right)^2\beta^2 D_k^2\tau_k^{\omega}}{8cL(cL+\kappa)V_{\max}^2} \right)\\
    &\quad\overset{\text{(i)}}{\geq} 1 - \sum_{k=0}^m |\mcs||\mca| \frac{2cL}{\kappa}\left(1+\frac{2cL}{\kappa}\right)\tau_{k}^\omega \cdot 2\exp\left( -\frac{\kappa^2\left( \frac{\Delta}{2+\Delta} \right)^2\beta^2 \epsilon^2\tau_k^{\omega}}{8cL(cL+\kappa)V_{\max}^2} \right)\\
    &\quad\overset{\text{(ii)}}{\geq} 1 - \frac{4cL}{\kappa}\left(1+\frac{2cL}{\kappa}\right)\sum_{k=0}^m |\mcs||\mca| \cdot \exp\left( -\frac{\kappa^2\left( \frac{\Delta}{2+\Delta} \right)^2\beta^2 \epsilon^2\tau_k^{\omega}}{16cL(cL+\kappa)V_{\max}^2} \right)\\
    &\quad\geq 1 - \frac{4cL(m+1)}{\kappa}\left(1+\frac{2cL}{\kappa}\right)|\mcs||\mca| \exp\left( -\frac{\kappa^2\left( \frac{\Delta}{2+\Delta} \right)^2\beta^2 \epsilon^2\tau_1^{\omega}}{16cL(cL+\kappa)V_{\max}^2} \right),
\end{align*}
where (i) follows because $D_k\geq D_m\geq \epsilon$, and (ii) follows from Lemma \ref{lem:tauHelp} by substituting $a=\frac{16cL(cL+\kappa)V_{\max}^2}{\kappa^2\left(\frac{\Delta}{2+\Delta}\right)^2\beta^2\epsilon^2 }, b=1$ and observing that
\begin{align*}
    \tau_k^{\omega}&\geq\hat{\tau}_1^{\omega}\geq \frac{32cL(cL+\kappa)V_{\max}^2}{\kappa^2\left(\frac{\Delta}{2+\Delta}\right)^2\beta^2\epsilon^2 }\ln\left(\frac{64cL(cL+\kappa)V_{\max}^2}{\kappa^2\left(\frac{\Delta}{2+\Delta}\right)^2\beta^2\epsilon^2 }\right)
    = 2ab\ln ab.
\end{align*}

\subsection{Part III: Bound $\norm{Q^A_t-Q^*}$}
In order to obtain the unconditional high-probability bound on $\norm{Q^A_t-Q^*}$, we first characterize a lower bound on the probability of Event $H$. Note that the probability of Event $G$ is lower bounded in~\Cref{lem:GqAsy}. 
\begin{lemma}\label{lem:halfQAAsy}
Let the sequence $\tau_k$ be the same as given in Lemma \ref{lem:coupleAsy}, i.e. $\tau_{k+1} = \tau_k + \frac{2cL}{\kappa}\tau_k^\omega$ for $k\geq 1$. Define $I^A_k$ as the number of iterations updating $Q^A$ at epoch $k$. Then we have
\begin{equation*}
    \mP\left[\forall k\in [1,m], I^A_k\geq cL\tau_{k}^\omega \right] \geq 1- m \exp\left( -\frac{(1-\kappa)^2cL\tau_1^\omega}{\kappa} \right).
\end{equation*}
\end{lemma}
\begin{proof}
We use the same idea as the proof of Lemma \ref{lem:halfQA}. Since we only focus on the blocks with $k\geq 1$, $I^A_k\sim Binomial\left( \frac{2cL}{\kappa}\tau_k^\omega, 0.5 \right)$ in such a case. Thus the tail bound of a binomial random variable gives
\begin{equation*}
    \mP \left[I^A_k\leq \frac{\kappa}{2}\cdot \frac{2cL}{\kappa}\tau_k^\omega\right] \leq \exp\left( -\frac{(1-\kappa)^2 cL\tau_k^\omega}{\kappa} \right).
\end{equation*}
Then by the union bound, we have
\begin{align*}
    \mP \left[\forall k\in[1,m], I^A_k\geq cL\tau_k^\omega\right]
    &=\mP \left[\forall k\in[1,m], I^A_k\geq \frac{\kappa}{2}\cdot \frac{2cL}{\kappa}\tau_k^\omega\right]\\
    &\geq 1 - \sum_{k=1}^m\exp\left( -\frac{(1-\kappa)^2 cL\tau_k^\omega}{\kappa} \right)\\
    &\geq 1- m \exp\left( -\frac{(1-\kappa)^2 cL\tau_1^\omega}{\kappa} \right).
\end{align*}
\end{proof}

Following from Lemma \ref{lem:totalIter}, it suffices to determine the starting time at epoch $m^*=\left\lceil \frac{4}{1-\gamma}\ln \frac{2\gamma V_{\max}}{\epsilon(1-\gamma)}\right\rceil$. This can be done by using Lemma \ref{lem:iteration} if we have $\hat{\tau}_1$.

Now we are ready to prove the main result of~\Cref{thm:asyncDQ}. By the definition of $m^*$, we know $D_{m^*-1}\geq\epsilon, G_{m^*-1}\geq \sigma\epsilon$. Then we obtain
\begin{align*}
    &\mP(\norm{Q^A_{\tau_{m^*}} - Q^* } \leq \epsilon)\\
    & \geq\mP\left[\forall k\in [0,m^* -1], \forall t\in[\tau_{k+1},\tau_{k+2}), \norm{Q^A_t  - Q^* }\leq D_{k+1} \right]\\
    & = \mP\left[ \forall k\in [0,m^* -1], \forall t\in[\tau_{k+1},\tau_{k+2}), \norm{Q^A_t  - Q^* }\leq D_{k+1} |G,H \right]\cdot\mP(G\cap H)\\
    & \geq \mP\left[ \forall k\in [0,m^* -1], \forall t\in[\tau_{k+1},\tau_{k+2}), \norm{Q^A_t  - Q^* }\leq D_{k+1} |G,H \right]\\
    &\quad \cdot (\mP(G)+\mP(H)-1)\\
    &\overset{\text{(i)}}{\geq}\mP\left[ \forall k\in [0,m^* -1], \forall t\in[\tau_{k+1},\tau_{k+2}), \norm{Q^A_t  - Q^* }\leq D_{k+1} |G,H \right]\\
    &\quad\cdot \left(\mP\left[ \forall q\in [0, m^* -1], \forall t\in[\hat{\tau}_{q+1},\hat{\tau}_{q+2}), \norm{Q^B_t  - Q^A_t }\leq G_{q+1} \right] + \mP(H) - 1\right)\\
    &\overset{\text{(ii)}}{\geq}\left[ 1 - \frac{4cLm^*}{\kappa}\left(1+\frac{2cL}{\kappa}\right)|\mcs||\mca| \exp\left( -\frac{\kappa^2\left( \frac{\Delta}{2+\Delta} \right)^2\beta^2 \epsilon^2\tau_1^{\omega}}{16cL(cL+\kappa)V_{\max}^2} \right) \right]\\
    &\quad\cdot\left[ 1- \frac{4cLm^*}{\kappa}\left(1+\frac{2cL}{\kappa}\right)|\mcs||\mca| \exp\left( -\frac{\kappa^2\left( \frac{\Delta}{2+\Delta} \right)^2\xi^2 \sigma^2\epsilon^2\hat{\tau}_1^{\omega}}{64cL(cL+\kappa)V_{\max}^2} \right)\right.\\
    &\quad\quad \left.- m^*\! \exp\!\left( -\frac{(1 -\kappa)^2 cL\hat{\tau}_1^{\omega}}{\kappa} \right) \right]\\
    &\geq 1- \frac{4cLm^*}{\kappa}\left(1+\frac{2cL}{\kappa}\right)|\mcs||\mca| \exp\left( -\frac{\kappa^2\left( \frac{\Delta}{2+\Delta} \right)^2\beta^2 \epsilon^2\tau_1^{\omega}}{16cL(cL+\kappa)V_{\max}^2} \right)\\
    &\quad - \!\frac{4cLm^*}{\kappa}\left(1+\frac{2cL}{\kappa}\right)|\mcs||\mca| \exp\left( -\frac{\kappa^2\left( \frac{\Delta}{2+\Delta} \right)^2\xi^2 \sigma^2\epsilon^2\hat{\tau}_1^{\omega}}{64cL(cL+\kappa)V_{\max}^2} \right)\! -\! m^* \exp\left( -\frac{(1\!-\!\kappa)^2 cL\hat{\tau}_1^{\omega}}{\kappa} \right)\\
    &\overset{\text{(iii)}}{\geq} 1- \frac{12cLm^*}{\kappa}\left(1+\frac{2cL}{\kappa}\right)|\mcs||\mca| \exp\left( -\frac{\kappa^2(1-\kappa)^2\left( \frac{\Delta}{2+\Delta} \right)^2\xi^2 \sigma^2\epsilon^2\hat{\tau}_1^{\omega}}{64 cL(cL+\kappa)V_{\max}^2} \right),
\end{align*}
where (i) follows from~\Cref{lem:coupleAsy}, (ii) follows from Propositions \ref{lem:GqAsy} and \ref{lem:conditionalBoundAsy} and (iii) holds due to the fact that
\begin{align*}
    \frac{4cLm^*}{\kappa}\left(1+\frac{2cL}{\kappa}\right)|\mcs||\mca| = \max&\left\{ \frac{4cLm^*}{\kappa}\left(1+\frac{2cL}{\kappa}\right)|\mcs||\mca|, m^* \right\},\\
    \frac{\kappa^2(1\!-\!\kappa)^2\left( \frac{\Delta}{2+\Delta} \right)^2\xi^2 \sigma^2\epsilon^2\hat{\tau}_1^{\omega}}{64cL(cL+\kappa)V_{\max}^2}\!\leq\! \min&\left\{ \frac{\kappa^2\!\left( \frac{\Delta}{2+\Delta} \right)^2\!\beta^2 \epsilon^2\hat{\tau}_1^{\omega}}{16cL(cL+\kappa)V_{\max}^2}, \frac{(1\!-\!\kappa)^2\hat{\tau}_1^{\omega}}{\kappa}, \frac{\kappa^2\!\left( \frac{\Delta}{2+\Delta} \right)^2\!\xi^2 \sigma^2\epsilon^2\hat{\tau}_1^{\omega}}{64cL(cL+\kappa)V_{\max}^2}\right\}.
\end{align*}

By setting 
\begin{equation*}
    1- \frac{12cLm^*}{\kappa}\left(1+\frac{2cL}{\kappa}\right)|\mcs||\mca| \exp\left( -\frac{\kappa^2(1-\kappa)^2\left( \frac{\Delta}{2+\Delta} \right)^2\xi^2 \sigma^2\epsilon^2\hat{\tau}_1^{\omega}}{64 cL(cL+\kappa)V_{\max}^2} \right) \geq 1-\delta,
\end{equation*}
we obtain
\begin{equation*}
    \hat{\tau}_1 \geq \left( \frac{64 cL(cL+\kappa)V_{\max}^2}{\kappa^2(1-\kappa)^2\left( \frac{\Delta}{2+\Delta} \right)^2\xi^2 \sigma^2\epsilon^2}\ln \frac{12m^*|\mcs||\mca|cL(2cL+\kappa)}{\kappa^2\delta} \right)^{\frac{1}{\omega}}.
\end{equation*}

Combining with the requirement of $\hat{\tau}_1$ in Propositions \ref{lem:GqAsy} and \ref{lem:conditionalBoundAsy}, we can choose 
\begin{equation*}
    \hat{\tau}_1 = \Theta\left( \left(\frac{ L^4V_{\max}^2}{(1-\gamma)^4\epsilon^2}\ln \frac{ m^*|\mcs||\mca|L^4V_{\max}^2}{(1-\gamma)^4\epsilon^2\delta} \right)^{\frac{1}{\omega}} \right).
\end{equation*}

Finally, applying $m^*=\left\lceil\frac{4}{1-\gamma}\ln \frac{2\gamma V_{\max}}{\epsilon(1-\gamma)}\right\rceil$ and Lemma \ref{lem:iteration}, we conclude that it suffices to let
\begin{align*}
    T&=\Omega\left( \left(\frac{ L^4V_{\max}^2}{(1-\gamma)^4\epsilon^2}\ln \frac{m^*|\mcs||\mca|L^4V_{\max}^2}{(1-\gamma)^4\epsilon^2\delta} \right)^{\frac{1}{\omega}} + \left(\frac{2cL}{\kappa}\frac{1}{1-\gamma} \ln\frac{\gamma V_{\max}}{(1-\gamma)\epsilon} \right)^{\frac{1}{1-\omega}} \right)\\
    &=\Omega\left( \left(\frac{ L^4V_{\max}^2}{(1-\gamma)^4\epsilon^2}\ln \frac{|\mcs||\mca|L^4V_{\max}^2 \ln \frac{2\gamma V_{\max}}{\epsilon(1-\gamma)}}{(1-\gamma)^5\epsilon^2\delta} \right)^{\frac{1}{\omega}} + \left(\frac{2cL}{\kappa}\frac{1}{1-\gamma} \ln\frac{\gamma V_{\max}}{(1-\gamma)\epsilon} \right)^{\frac{1}{1-\omega}} \right)\\
    &=\Omega\left( \left(\frac{ L^4V_{\max}^2}{(1-\gamma)^4\epsilon^2}\ln \frac{|\mcs||\mca|L^4V_{\max}^2 }{(1-\gamma)^5\epsilon^2\delta} \right)^{\frac{1}{\omega}} + \left(\frac{L^2}{1-\gamma} \ln\frac{\gamma V_{\max}}{(1-\gamma)\epsilon} \right)^{\frac{1}{1-\omega}} \right).
\end{align*}
to attain an $\epsilon$-accurate Q-estimator.